\documentclass{article}

\usepackage[preprint]{neurips_2023}

\usepackage[utf8]{inputenc} 
\usepackage[T1]{fontenc}    
\usepackage{hyperref}       
\usepackage{url}            
\usepackage{booktabs}       
\usepackage{amsfonts}       
\usepackage{nicefrac}       
\usepackage{microtype}      

\usepackage[subrefformat=parens]{subcaption}
\usepackage[dvips]{graphicx}
\usepackage{amssymb}
\usepackage{amstext}
\usepackage{amsmath}
\usepackage{amscd}
\usepackage{amsthm}
\usepackage{amsfonts}
\usepackage{enumerate}
\usepackage{graphicx}
\usepackage{latexsym}
\usepackage{mathrsfs}
\usepackage{mathtools}
\usepackage{cases}
\usepackage{verbatim}
\usepackage{bm}
\usepackage{tikz}
\usepackage{caption}
\usepackage{authblk}
\usepackage{comment}
\usepackage{natbib}
\usepackage{arydshln}
\usepackage{blkarray}

\def\ba{\mathbf{a}}

\def\cL{\mathcal{L}}

\newcommand{\R}{\mathbb{R}}

\newcommand{\innerprod}[2]{\left \langle#1,#2\right\rangle}

\newcommand{\relu}{\operatorname{ReLU}}

\newcommand{\pmat}[1]{\begin{pmatrix} #1 \end{pmatrix}}
\newcommand{\paren}[1]{\left ( #1\right)}

\makeatletter
\renewcommand{\boxed}[1]{\text{\fboxsep=.2em\fbox{\m@th$\displaystyle#1$}}}
\makeatother

\usepackage{amsthm}
\newtheorem{lemma}{Lemma}
\newtheorem{theorem}{Theorem}

\newtheorem{proposition}{Proposition}
\newtheorem{definition}{Definition}
\newtheorem{example}{Example}

\newtheorem{remark}{Remark}

\usepackage{graphicx}
\usepackage{url}
\usepackage{amsfonts}
\usepackage{amssymb}
\usepackage{a4}
\usepackage [all] {xy}

\usepackage{hyperref} 

\usepackage{color}

\numberwithin{equation}{section}

\def \beq {\begin {eqnarray}}
\def \eeq {\end {eqnarray}}
\def \ba {\begin {eqnarray*}}
\def \ea {\end  {eqnarray*}}

\renewcommand{\tilde}{\widetilde}

\def\tilde{\widetilde}
\def \bfo {\begin {eqnarray*} }
\def \efo {\end {eqnarray*} }
\def \ba {\begin {eqnarray*} }
\def \ea {\end {eqnarray*} }
\def \beq {\begin {eqnarray}}
\def \eeq {\end {eqnarray}}

\def\bra{\langle}

\def\ket{\rangle}
\def \e {\varepsilon}

\def \p {\partial}

\def\tilde{\widetilde}
\def \bfo {\begin {eqnarray*} }
\def \efo {\end {eqnarray*} }
\def \ba {\begin {eqnarray*} }
\def \ea {\end {eqnarray*} }
\def \beq {\begin {eqnarray}}
\def \eeq {\end {eqnarray}}

\def\bra{\langle}

\def\ket{\rangle}

\def \e {\varepsilon}
\def \p {\partial}

\title{\textbf{Globally injective and bijective neural operators}}

\author[1]{\rm Takashi Furuya}

\author[2]{\rm Michael Puthawala}

\author[3]{\rm Matti Lassas}

\author[4]{\rm Maarten V. de Hoop}

\affil[1]{{\small Shimane University}}
\affil[ ]{{\small takashi.furuya0101@gmail.com}\vspace{3mm}}

\affil[2]{{\small South Dakota State University}}
\affil[ ]{{\small Michael.Puthawala@sdstate.edu}\vspace{3mm}}

\affil[3]{{\small University of Helsinki}}
\affil[ ]{{\small matti.lassas@helsinki.fi}\vspace{3mm}}

\affil[4]{{\small Rice University}}
\affil[ ]{{\small mdehoop@rice.edu}\vspace{3mm}}

\begin{document}
\maketitle
\begin{abstract}
    Recently there has been great interest in operator learning, where networks learn operators between function spaces from an essentially infinite-dimensional perspective. In this work we present results for when the operators learned by these networks are injective and surjective. As a warmup, we combine prior work in both the finite-dimensional ReLU and operator learning setting by giving sharp conditions under which ReLU layers with linear neural operators are injective. We then consider the case the case when the activation function is pointwise bijective and obtain sufficient conditions for the layer to be injective. We remark that this question, while trivial in the finite-rank case, is subtler in the infinite-rank case and is proved using tools from Fredholm theory. Next, we prove that our supplied injective neural operators are universal approximators and that their implementation, with finite-rank neural networks, are still injective. This ensures that injectivity is not `lost' in the transcription from analytical operators to their finite-rank implementation with networks. Finally, we conclude with an increase in abstraction and consider general conditions when subnetworks, which may be many layers deep, are injective and surjective and provide an exact inversion from a `linearization.’ This section uses general arguments from Fredholm theory and Leray-Schauder degree theory for non-linear integral equations to analyze the mapping properties of neural operators in function spaces. These results apply to subnetworks formed from the layers considered in this work, under natural conditions. We believe that our work has applications in Bayesian UQ where injectivity enables likelihood estimation and in inverse problems where surjectivity and injectivity corresponds to existence and uniqueness, respectively.
\end{abstract}

\section{Introduction}\label{Introduction}

In this work, we produce results at the intersection of two fields: neural operators, and injective \& bijective networks. Neural operators \citep{kovachki2021universal,kovachki2021neural} are neural networks that take a infinite dimensional perspective on approximation by directly learning a operators between Sobolev spaces. Injectivity and bijectivity on the other hand are fundamental properties of networks that enable likelihood estimate by the change of variables formula, a useful property for downstream applications.

Put simply, the key contribution of our work is the translation of fundamental notions from the finite-rank setting to the infinite-rank setting. By the infinite-dimension setting we refer to the case when the object of approximation is a mapping between infinite-rank Sobolev spaces. This task, although straight-forward on its face, often requires dramatically different arguments and proofs as the topology, analysis and notion of noise are much simpler in the finite-rank case as compared to the infinite-rank case. We see our work as laying the groundwork for the application of neural operators to generative models in function spaces. In the context of operator extensions of traditional VAEs \citep{kingma2013auto}, injectivity of a decoder forces distinct latent codes to correspond to distinct outputs.

\subsection{Our contribution}
In this paper, we give a rigorous framework for the analysis of the injectivity and bijectivity of neural operators. Our contributions are follows:

\begin{itemize}

\item[(A)] 
We show the equivalent condition for the layerwise injectivity and bijectivity for linear neural operators in the case of ReLU and bijective activation functions (Section~\ref{Layerwise-injectivity}).
In particular of ReLU case, the equivalent condition is characterized by directed spanning set (Definition~\ref{definition:DSS:general}).

\item[(B)] 
We prove that injective linear neural operators are universal approximators, and their implementation by finite-rank approximation is still injective (Section~\ref{Global-injectivity}).
We note that universality approximation theorem (Theorem~\ref{Universal-injectivity-NO}) in infinite dimensional case does not require an increase in dimension, which deviate from finite dimensional case \citet[Thm.\ 15]{puthawala2022globally}.

\item[(C)] 
We zoom out and perform a more abstract global analysis in the case when the input and output dimension are the same. In this section we `coarsen' the notion of layer, and provide the sufficient condition for the surjectivity and bijectivity for nonlinear integral neural operators \emph{with nonlinear kernels}. This application arises naturally in the context of subnetworks and transformers. We construct their inverses in the bijective case (Section~\ref{Non-linear operator}).

\end{itemize}

\subsection{Related works}

In the finite-rank setting, injective networks have been well-studied, and shown to be of theoretical and practical interest. See \citet{gomez2017reversible,kratsios2020non,puthawala2022globally} for general references establishing the usefulness of injectivity or any of the works on flow networks for the utility of injectivity and bijectivity for downstream applications, \citep{dinh2016density,siahkoohi2020faster,chen2019residual,dinh2014nice,kingma2016improving}.  but their study in the infinite-rank setting is comparatively underdeveloped. These works, and others, establish injectivity in the finite-rank setting as a property of theoretical and practical interest. Our work extend \citet{puthawala2022globally} to the infinite-dimensional setting as applied to neural operators, which themselves are a generalization of multilayer perceptrons (MLPs) to function spaces.

Examples of works in these setting include neural operators \citet{kovachki2021neural, kovachki2021universal}, DeepONet \citet{lu2019deeponet, lanthaler2022error}, and PCA-Net \citet{bhattacharya2021model, de2022cost}. 
The authors of \citet{alberti2022continuous} recently proposed continuous generative neural networks (CGNNs), which are convolution-type architectures for generating $L^2(\mathbb{R})$-functions, and provided the sufficient condition for the global injectivity of their network. 
Their approach is the wavelet basis expansion, whereas our work relies on an independent choice of basis expansion.

\subsection{Networks considered and notation}\label{Networks considered and notation}

Let $D \subset \mathbb{R}^{d}$ be an open and connected domain, and $L^{2}(D; \mathbb{R}^{h})$ be the $L^2$ space of $\mathbb{R}^{h}$-value function on $D$ given by
\[
L^{2}(D; \mathbb{R}^{h}) 
\coloneqq 
\underbrace{
L^{2}(D;\mathbb{R}) \times \cdots \times L^{2}(D;\mathbb{R})
}_{h}=L^{2}(D)^{h}.
\]
\begin{definition}[Integral and pointwise neural operators]
We define layers $\mathcal{L}_{\ell} : L^{2}(D)^{d_{\ell}} \to L^{2}(D)^{d_{\ell+1}}$ and integral neural operator $G:L^{2}(D)^{d_{in}} \to L^{2}(D)^{d_{out}}$ by
\[
\label{eqn:g-def}
x \in D, \quad (\mathcal{L}_{\ell}v)(x) \coloneqq \sigma (T_{\ell}(v)(x) + b_{\ell}(x)),\quad G :=  T_{L+1} \circ \mathcal{L}_{L} \circ 
\cdots \mathcal{L}_{1} \circ T_{0}
\]
where $\sigma : \mathbb{R} \to \mathbb{R}$ is a non-linear activation operating element-wise, and $T_{\ell}:L^{2}(D)^{d_{\ell}} \to L^{2}(D)^{d_{\ell+1}}$ are sums of linear integral operators having kernels $k_{\ell}(x,y)\in L^2(D\times D)$ and pointwise multiplications with matrices,
$W_{\ell}:u(x) \mapsto W_{\ell}(x)u(x)$, $W_{\ell} \in C(\overline D; \mathbb{R}^{d_{\ell}\times d_{\ell+1}})$, and $b_{\ell} \in L^{2}(D)^{d_{\ell+1}}$ are bias functions  ($\ell = 0,...,L+1$).
Here, $T_{0}: L^{2}(D)^{d_{in}} \to L^{2}(D)^{d_{1}}$ and $T_{L+1} : L^{2}(D)^{d_{L+1}} \to L^{2}(D)^{d_{out}}$ are mappings (lifting operator) from the input space to the feature space and mappings (projection operator) from the feature spaces to the output space, respectively.
\end{definition}

The layers $T_{0}$ and $T_{L+1}$ play a special role in the neural operators. They are local linear operators and serve to lift and project the input data from and to finite-dimensional space respectively. These layers may be `absorbed' into the layers $\cL_1$ and $\cL_{L}$ without loss of generality (under some technical conditions), but are not in this text to maintain consistency with prior work.
Prior work assumes that $d_{in} < d_{1}$, we only assume that $d_{in}\leq d_{1}$ for lifting operator $T_{0}: L^{2}(D)^{d_{in}} \to L^{2}(D)^{d_{1}}$. This would seemingly play an important role in the context of injectivity or universality, but we find that our analysis does not require that $d_{in}<d_{1}$ at all. In fact, as elaborated in Section~\ref{Universal approximation}, we may take $d_{in}= d_{\ell} = d_{out}$ for $\ell = 1,\dots,L$ and our analysis is the same.

\section{Injective linear neural operator layers with ReLU and bijective activations}\label{Layerwise-injectivity}

In this section we first present sharp conditions under which a layer of a neural operator with $\relu$ activation is injective. The Directed Spanning Set (DSS) condition, described by Def. \ref{definition:DSS:general} is a generalization of the finite-dimensional DSS \citep{puthawala2022globally} which guarantees layerwise injectivity of $\relu$ layers. Our obtained condition is quite restrictive hence we also consider the case when a bijective activation function is used. Although this case is trivial in the finite-dimensional setting, it becomes subtle in the infinite-rank case. 

We denote by
\[
\sigma(Tv+b)(x):=\pmat{\sigma(T_1v(x)+b_1(x))\\\vdots\\\sigma(T_mv(x)+b_{m}(x))}, \quad x \in D, \ v \in L^{2}(D)^{n}
\]
where $\sigma : \mathbb{R} \to \mathbb{R}$ is a non-linear activation function, $T \in \mathcal{L}(L^{2}(D)^{n}, L^{2}(D)^{m})$, and $b \in L^{2}(D)^{m}$, where $\mathcal{L}(L^{2}(D)^{n}, L^{2}(D)^{m})$ is the space of linear bounded operators from $L^{2}(D)^{n}$ to $L^{2}(D)^{m}$.
The aim of this section is to characterize the injectivity condition for the operator $v \mapsto \sigma(Tv+b)$ mapping from $L^{2}(D)^n$ to $L^{2}(D)^{m}$, which corresponds to layer operators $\mathcal{L}_{\ell}$. Here, $T: L^{2}(D)^{n} \to L^{2}(D)^{m}$ is linear.

\subsection{ReLU activation}\label{ReLU-case}
Let $\sigma : \mathbb{R} \to \mathbb{R}$ be ReLU activation, defined by $\relu(s)=\max\{ 0, s \}$. With this activation function, we introduce a definition which we will find sharply characterizes layerwise injectivity. 
\begin{definition}[Directed Spanning Set]\label{definition:DSS:general}
We say that the operator $T+b$ has a directed spanning set (DSS) with respect to $v\in L^{2}(D)^{n}$ if 
\begin{equation}
\mathrm{Ker}\left(T\bigl|_{S(v,T+b)}\right) \cap X(v, T+b) = \{  0\},\label{zero-set-DSS}
\end{equation}
where $T|_{S(v,T+b)}(v)=(T_iv)_{i\in S(v,T+b)}$ and
\begin{equation}
S(v,T+b) := \left\{i \in [m] \ \Big| \ T_{i}v + b_{i} > 0 \ in \ D \right\},
\end{equation}
\begin{equation}
X(v,T+b) := \left\{
u \in L^{2}(D)^{n} \Biggl| 
\begin{array}{ll}
\mbox{for} \ i \notin S(v,T+b)\hbox{ and } x\in D, \\
(i)\ T_{i}v(x)+b_{i}(x) \leq T_{i}u(x) \text{ if } T_{i}v(x)+b_{i}(x) \leq 0, \\ 
(ii)\ T_{i}u(x) = 0 \text{ if } T_{i}v(x)+b_{i}(x) > 0
\end{array}
\right\}.
\end{equation}
\end{definition}

\begin{proposition}
\label{injectivity-ReLU:general}
Let $T \in \mathcal{L}(L^{2}(D)^{n}, L^{2}(D)^{m})$ and $b \in  L^{2}(D)^{m}$.
Then, the operator $\relu \circ (T+b) : L^{2}(D)^{n}\to L^{2}(D)^{m}$ is injective if and only if $T+b$ has a DSS with respect to every $v \in L^{2}(D)^{n}$ in the sense of Definition~\ref{definition:DSS:general}.
\end{proposition}

See Section~\ref{proof-injectivity-ReLU:general} for the proof. \citet{puthawala2022globally} has provided the equivalent condition for the injectivity of ReLU operator in the case of the Euclidean space. 
However, proving analogous results for operators in function spaces require different techniques. Note that because Def. \ref{definition:DSS:general} is a sharp characterization of injectivity, it can not be simplified in any significant way. 
The condition restrictive Def. \ref{definition:DSS:general} is, therefore, difficult to relax while maintaining generality. This is because for each function $v$, multiple components of the function $Tv+b$ are strictly positive in the entire domain $D$, and cardinality $|S(v;T+b)|$ of $S(v;T+b)$ is larger than $n$.
This observation prompts us to consider the use of bijective activation functions instead of ReLU, such as leaky ReLU function, defined by $\sigma_a(s):=\relu(s)-a\relu(-s)$ where $a > 0$.

\subsection{Bijective activation}\label{bijective activation linear}

If $\sigma$ is injective, then injectivity of $\sigma \circ (T+b) : L^{2}(D)^{n}\to L^{2}(D)^{m}$ is equivalent to the injectivity of $T$.
Therefore, we consider the bijectivity in the case of $n=m$.
As mentioned in Section~\ref{Networks considered and notation}, an significant example is  $T=W+K$, where $W \in \mathbb{R}^{n\times n}$ is injective and $K:L^{2}(D)^n \to L^{2}(D)^n$ is linear integral operator with a smooth kernel.
This can be generalized to Fredholm operator (see e.g., \citet[Section 2.1.4]{jeribi2015spectral}), which encompass the property for identity plus compact operator. 
It is well known that 
a Fredholm operator is bijective if and only if it is injective and its Fredholm index is zero. We summarize the above observation as follows:
\begin{proposition}
\label{injective-activation-chara}
Let $\sigma : \mathbb{R} \to \mathbb{R}$ be bijective, and let $T : L^{2}(D)^{n}\to L^{2}(D)^{m}$ and $b \in L^{2}(D)^{m}$.
Then, $\sigma \circ (T+b) : L^{2}(D)^{n}\to L^{2}(D)^{m}$ is injective if and only if $T: L^{2}(D)^{n}\to L^{2}(D)^{m}$ is injective.
Furthermore, if $n=m$ and $T \in \mathcal{L}(L^{2}(D)^{n}, L^{2}(D)^{n})$ is the linear Fredholm operator, then, $\sigma \circ (T+b) : L^{2}(D)^{n}\to L^{2}(D)^{n}$ is bijective if and only if $T: L^{2}(D)^{n}\to L^{2}(D)^{n}$ is injective with index zero.
\end{proposition}
We believe that this characterization of layerwise injectivity is considerably less restrictive than Def. \ref{definition:DSS:general}, and the characterization of bijectivity in terms of Fredholm theory will be particularly useful in establishing operator generalization of flow networks.

\section{Global analysis of injectivity and finite-rank implementation}
\label{Global-injectivity}

In this section we consider global properties of the injective and bijective networks that constructed in Section \ref{Layerwise-injectivity}. First we construct end-to-end injective networks that are not layerwise injective. By doing this, we may avoid the dimension increasing requirement that would be necessary from a layerwise analysis. Next we show that injective neural operators are universal approximators of continuous functions. Although the punchline resembles that of \cite[Theorem 15]{puthawala2022globally}, which relied on Whitney's embedding theorem, the arguments are quite different. The finite-rank case has dimensionality restrictions, as required by degree theory, whereas our infinite-rank result does not. Finally, because all implementations of neural operators are ultimately finite-dimensional, we present a theorem that gives conditions under which finite-rank approximations to injective neural operators are also injective.

\subsection{Global analysis}\label{Global analysis}
By using the characterization of layerwise injectivity discussed in Section~\ref{Layerwise-injectivity}, we can compose injective layers to form $\mathcal{L}_{L} \circ 
\cdots \mathcal{L}_{1} \circ T_{0}$, a injective network.
Layerwise injectivity, however, prevents us from getting injectivity of $T_{L+1}\circ \mathcal{L}_{L} \circ 
\cdots \mathcal{L}_{1} \circ T_{0}$ by a layerwise analysis if $d_{L+1} > d_{out}$, as is common in application \cite[Pg. 9]{kovachki2021neural}.   
In this section, we consider global analysis and show that $T_{L+1}\circ \mathcal{L}_{L} \circ 
\cdots \mathcal{L}_{1} \circ T_{0}$, nevertheless, remains injective. This is summarized in the following lemma.
\begin{lemma}\label{injective-dimensional-reduction}
Let $\ell \in \mathbb{N}$ with $\ell < m$, and let the operator $T:L^{2}(D)^{n} \to L^{2}(D)^{m}$ be injective. 
Assume that there exists an orthogonal sequence $\{ \xi_{k} \}_{k \in \mathbb{N}}$ in $L^{2}(D)$ such that 
\begin{equation}
\mathrm{span}\{\xi_{k}\}_{k \in \mathbb{N}} \cap \mathrm{Ran}(\pi_{1} T) = \{ 0 \}.
\label{ONS-notin-range-main}
\end{equation}
where $\pi_{1}: L^{2}(D)^{m} \to L^{2}(D)$ is the restriction operator defined in (\ref{restriction-pi}).
Then, there exists a linear bounded operator $B \in \mathcal{L}(L^{2}(D)^{m}, L^{2}(D)^{\ell})$ such that 
$
B \circ T : L^{2}(D)^{n} \to L^{2}(D)^{\ell}
$
is injective.
\end{lemma}
See Section~\ref{proof-injective-dimensional-reduction} in Appendix~\ref{Appendix2} for the proof.
$T$ and $B$ correspond to $\mathcal{L}_{L} \circ 
\cdots \mathcal{L}_{1} \circ T_{0}$ (from lifting to $L$-th layer) and $T_{L+1}$ (projection), respectively.
The assumption~(\ref{ONS-notin-range-main}) on the span of $\xi_k$ encodes a subspace distinct from the range of $T$. 
In Remark~\ref{remark-choice-xi} of Appendix~\ref{Appendix2}, we provide an example that satisfies the assumption~(\ref{ONS-notin-range-main}).
Moreover, in Remark~\ref{existential-B} of Appendix~\ref{Appendix2}, we show the exact construction of the operator $B$ by employing projections onto the closed subspace, using the orthogonal sequence $\{ \xi_{k} \}_{k \in \mathbb{N}}$.
This construction is given by the combination of "Pairs of projections" discussed in \citet[Section~I.4.6]{kato2013perturbation} with the idea presented in \cite[Lemma 29]{puthawala2022universal}.

\subsection{Universal approximation}\label{Universal approximation}

We now show that the injective networks that we consider in this work are universal approximators. Loosely, the integral neural operators have $L^2$ integral kernels and are a fixed depth, with full definition given below. 
\begin{definition}\label{definition-NN-A-BA-NO}
The set of $L$-layer neural networks mapping from $\mathbb{R}^{d}$ to $\mathbb{R}^{d^{\prime}}$ is
\[
\begin{split}
\mathrm{N}_{L}(\sigma; \mathbb{R}^{d}, \mathbb{R}^{d^{\prime}})
& := \Bigl\{ 
f:\mathbb{R}^{d} \to \mathbb{R}^{d^{\prime}}  \Big|
f(x)=W_{L}\sigma(\cdots W_{1}\sigma(W_{0}x+b_{0})+b_{1}\cdots )+b_{L},
\\
&
W_{\ell} \in \mathbb{R}^{d_{\ell+1}\times d_{\ell}},
b_{\ell} \in \mathbb{R}^{d_{\ell+1}},
d_{\ell} \in \mathbb{N}_{0}
(d_{0}=d, \ d_{L+1}=d^{\prime}), 
\ell=0,...,L
\Bigr\},
\end{split}
\]
where $\sigma : \mathbb{R}\to \mathbb{R}$ is an element-wise nonlinear activation function. 
For the class of nonlinear activation functions,
\[
\mathrm{A}_{0}:=\Bigl\{
\sigma \in C(\mathbb{R}) \Big|
\exists n \in \mathbb{N}_{0} \text{ s.t. } \mathrm{N}_{n}(\sigma; \mathbb{R}^{d}, \mathbb{R}) \text{ is dense in } C(K) \text{ for } \forall K \subset \mathbb{R}^{d} \text{ compact}
\Bigr\}
\]
\[
\mathrm{A}_{0}^{L}:=\Bigl\{
\sigma \in A_{0} \Big|
\sigma \text{ is Borel measurable  s.t. } \sup_{x \in \mathbb{R}} \frac{|\sigma(x)|}{1+|x|} < \infty 
\Bigr\}
\]
\[
\begin{split}
\mathrm{BA}
& :=\Bigl\{
\sigma \in A_{0} \Big|
\forall K \subset \mathbb{R}^{d} \text{ compact }, \forall \epsilon>0, \text{ and } \forall C \geq \mathrm{diam}(K),
\exists n \in \mathbb{N}_{0}, 
\\
&
\exists f \in \mathrm{N}_{n}(\sigma; \mathbb{R}^{d}, \mathbb{R}^{d}) \text{ s.t. }
|f(x)-x| \leq \epsilon, \ \forall x \in K, \text{ and, }
|f(x)| \leq C, \ \forall x \in \mathbb{R}^{d}
\Bigr\}.
\end{split}
\]
The set of integral neural operators with $L^{2}$-integral kernels is
\begin{equation}
\begin{split}
 \mathrm{NO}_{L}&
 (\sigma;  D, d_{in}, d_{out})
:= \Bigl\{ 
G:L^{2}(D)^{d_{in}} \to L^{2}(D)^{d_{out}} 
\Big|
\\
&
G= K_{L+1} \circ (K_{L}+b_{L}) \circ \sigma
\cdots \circ (K_{2}+b_{2}) \circ \sigma \circ (K_{1}+b_{1}) \circ (K_{0}+b_{0}),
\\
&
K_{\ell} \in \mathcal{L}(L^{2}(D)^{d_{\ell}}, L^{2}(D)^{d_{\ell+1}}), \
K_{\ell}: f \mapsto \int_{D}k_{\ell}(\cdot, y)f(y)dy \biggl|_{D}, 
\\
&
k_{\ell} \in L^2(D \times D; \mathbb{R}^{d_{\ell+1} \times d_{\ell} }), \ 
b_{\ell} \in 
 L^2(D; \mathbb{R}^{d_{\ell+1} }),
\\
&
d_{\ell} \in \mathbb{N}, \
d_{0}=d_{in}, \ d_{L+2}=d_{out}, 
\
\ell=0,...,L+2
\Bigr\}.
\end{split}
\label{definition-NO_L}
\end{equation}
\end{definition}

The following theorem shows that $L^2$-\emph{injective} neural operators are universal approximators of continuous operators.
\begin{theorem}\label{Universal-injectivity-NO}
Let $D\subset \mathbb{R}^{d}$ be a Lipschitz bounded domain, and $G^{+}:L^{2}(D)^{d_{in}} \to L^{2}(D)^{d_{out}}$ be continuous such that for all $R>0$ there is $M>0$ so that
\begin{equation}
\left\| G^{+}(a)\right\|_{L^{2}(D)^{d_{out}}}\leq M, \ 
\forall a \in L^{2}(D)^{d_{in}},\ \|a\|_{L^{2}(D)^{d_{in}}}\leq R,
\label{bound-R-G-M}
\end{equation}
We assume that either (i) $\sigma \in \mathrm{A}_{0}^{L} \cap \mathrm{BA}$ is injective, or (ii) $\sigma = \relu$.
Then, for any compact set $K \subset L^{2}(D)^{d_{in}}$, $\epsilon \in (0,1)$, there exists $L \in \mathbb{N}$ and $G \in \mathrm{NO}_{L}(\sigma; D, d_{in}, d_{out})$ such that
$G:L^{2}(D)^{d_{in}} \to L^{2}(D)^{d_{out}}$ is injective and
\[
\sup_{a \in K} \left\|G^{+}(a) - G(a) \right\|_{L^{2}(D)^{d_{out}}} \leq \epsilon.
\]
\end{theorem}

See Section~\ref{proof-Universal-injectivity-NO} in Appendix~\ref{Appendix2} for the proof. We remark that both ReLU and Leaky ReLU functions belong to $\mathrm{A}_{0}^{L} \cap \mathrm{BA}$ (see Remark~\ref{remark-L2-universality} (i)).

We briefly remark on the proof of Theorem~\ref{Universal-injectivity-NO} emphasizing how its proof differs from a straight-forward extension of the finite-rank case. In the proof we first employ the standard universality approximation theorem for neural operators (\citep[Theorem 11]{kovachki2021neural}). 
We denote the approximation of $G^{+}$ by $\widetilde{G}$, and
define the graph of $\widetilde{G}$ as $H:L^2(D)^{d_{in}} \to L^2(D)^{d_{in}}\times L^2(D)^{d_{out}}$. That is $H(u)=(u, \widetilde{G}(u))$. 
Next, utilizing Lemma~\ref{injective-dimensional-reduction}, we construct the projection $Q$ such that $Q \circ H$ becomes an injective approximator of $G^{+}$ and belongs to $\mathrm{NO}_{L}(\sigma;  D, d_{in}, d_{out})$.
This approach resembles the approach in the finite-rank space \citet[Theorem 15]{puthawala2022globally}, but unlike that theorem we don't have any dimensionality restrictions. More specifically, in the case of Euclidean spaces $\R^d$, \citet[Theorem 15]{puthawala2022globally} requires that $2d_{in}+1 \leq d_{out}$ before all continuous functions $G^+:\R^{d_{in}}\to \R^{d_{out}}$ can be uniformly approximated in compact sets by injective neural networks.

When $d_{in}=d_{out}=1$ this result is not true, as is shown in Remark~\ref{remark-L2-universality} (iii)
in Appendix~\ref{Appendix2} using topological degree theory \citep{cho2006topological}. 
In contrast, Theorem~\ref{Universal-injectivity-NO} does not assume any conditions on $d_{in}$ and $d_{out}$. 
Therefore, we can conclude that infinite dimensional case yields better approximation results than the finite dimensional case.

This surprising improvement in restrictions in infinite-dimensions can be elucidate by 
an analogy to Hilbert's hotel paradox, see
\cite[Sec 3.2]{burger2004heart}. 
In this analogy, the orthonormal bases $\{\varphi_{k}\}_{k \in \mathbb{N}}$ and $\Psi_{j,k}(x)=(\delta_{ij}\varphi_{k}(x))_{i=1}^d$ play the part of guests in the hotel with $\mathbb N$ floor, each of which as $d$ rooms. A key step in the proof of Theorem~\ref{Universal-injectivity-NO} 
is that there is a linear isomorphism $S:L^2(D)^d\to L^2(D)$ (i.e., a rearrangement of guests) which maps $\Psi_{j,k}$ to  $\varphi_{b(j,k)}$, where
$b:[d]\times \mathbb{N}\to \mathbb{N}$ is a bijection.

\subsection{Injectivity-preserving transfer to Euclidean spaces via finite-rank approximation}
\label{Approximation and injectivity via finite-rank operators}
In the previous section, we have discussed injective integral neural operators.
The conditions are given in terms of integral kernels, but such kernels may not actually be implementable with finite width and depth networks, which have a finite representational power. A natural question to ask, therefore, is how should these formal integral kernels be implemented on actual finite-rank networks, the so-called implementable neural operators? In this section we discuss this question.

We consider linear integral operators $K_{\ell}$ with $L^2$-integral kernels $k_{\ell}(x,y)$. 
Let $\{\varphi_{k}\}_{k \in \mathbb{N}}$ be an orthonormal basis in $L^{2}(D)$. 
Since $\{\varphi_{k}(y)\varphi_{p}(x)\}_{k,p \in \mathbb{N}}$ is an orthonormal basis of $L^{2}(D \times D)$, integral kernels $k_{\ell} \in L^2(D \times D; \mathbb{R}^{d_{\ell+1} \times d_{\ell} })$ in integral operators $K_{\ell} \in \mathcal{L}(L^{2}(D)^{d_{\ell}}, L^{2}(D)^{d_{\ell+1}})$ has the expansion
\[
k_{\ell}(x,y) = \sum_{k,p \in \mathbb{N}}C^{(\ell)}_{k,p}\varphi_{k}(y)\varphi_{p}(x),
\]
where $C^{(\ell)}_{k,p} \in \mathbb{R}^{d_{\ell+1} \times d_{\ell}}$ whose $(i,j)$-th component $c^{(\ell)}_{k,p,ij}$ is given by
$$
c^{(\ell)}_{k,p,ij}=(k_{\ell, ij}, \varphi_k \varphi_p)_{L^{2}(D\times D)},
$$
Here, we denote $(u, \varphi_{k}) \in \mathbb{R}^{d_{\ell}}$ by
$
(u, \varphi_{k}) =\left( (u_1, \varphi_{k})_{L^{2}(D)}, ..., (u_{d_{\ell}}, \varphi_{k})_{L^{2}(D)} \right).
$
By truncating by $N$ finite sums, we approximate $L^2$-integral operators $K_{\ell} \in \mathcal{L}(L^{2}(D)^{d_{\ell}}, L^{2}(D)^{d_{\ell+1}})$ by finite-rank operator $K_{\ell, N}$ with rank $N$, having the form
\[
K_{\ell,N}u(x) = \sum_{k,p \in [N]}C^{(\ell)}_{k,p}(u, \varphi_{k}) \varphi_{p}(x), \ u \in L^{2}(D)^{d_{\ell}}.
\]
The choice of orthonormal basis $\{\varphi_{k}\}_{k \in \mathbb{N}}$ is a hyperparameter. 
If we choose $\{ \varphi_{k} \}_{k}$ as Fourier basis and wavelet basis, then network architectures correspond to FNOs \citep{li2020fourier} and  
WNOs \citep{tripura2023wavelet}, respectively.

We show that Propositions~\ref{injectivity-ReLU:general}, \ref{injective-activation-chara} (characterization of layerwise injectivity), and  Lemma~\ref{injective-dimensional-reduction} (global injectivity) all have natural analogues for finite-rank operator $K_{\ell, N}$ in Proposition~\ref{injectivity-ReLU:finite-rank} and Lemma~\ref{injective-dimensional-reduction-finite-rank} in Appendix~\ref{Appendix3}.
We also show the universal approximation in the case of finite-rank approximation.

We now formally define the set of finite-rank neural operators.

\begin{definition}\label{definition-finite-rank-NO-set}
We define the set of integral neural operators with $N$ rank by
\begin{equation}
\begin{split}
& \mathrm{NO}_{L, N}
 (\sigma;  D, d_{in}, d_{out})
:= \Bigl\{ 
G_{N} : L^{2}(D)^{d_{in}} \to L^{2}(D)^{d_{out}} 
\Big|
\\
&
G_{N}= K_{L+1,N} \circ (K_{L,N}+b_{L,N}) \circ \sigma
\cdots \circ (K_{2, N}+b_{2,N}) \circ \sigma \circ (K_{1,N}+b_{1,N}) \circ (K_{0,N}+b_{0,N}),
\\
&
K_{\ell, N} \in \mathcal{L}(L^{2}(D)^{d_{\ell}}, L^{2}(D)^{d_{\ell+1}}), \
K_{\ell,N}: f \mapsto \sum_{k,p \leq N}C^{(\ell)}_{k,p}(f, \varphi_{k}) \varphi_{p}, 
\\
&
b_{\ell,N} \in L^2(D; \mathbb{R}^{d_{\ell+1}}), \
b_{\ell,N} =\sum_{p \leq N} b^{(\ell)}_{p} \varphi_{m}
\\
&
C^{(\ell)}_{k,p} \in \mathbb{R}^{d_{\ell +1} \times d_{\ell}}, \
b^{(\ell)}_{p} \in \mathbb{R}^{d_{\ell+1}}, \ k,p \leq N,
\\
&
d_{\ell} \in \mathbb{N}, \
d_{0}=d_{in}, \ d_{L+2}=d_{out}, 
\
\ell=0,...,L+2
\Bigr\}.
\end{split}
\end{equation}
\end{definition}

With this notion of finite-rank neural operators, we obtain the following theorem.
\begin{theorem}\label{Universal-injectivity-NO-finite-rank}
Let $D\subset \mathbb{R}^{d}$ be a Lipschitz bounded domain, and $N \in \mathbb{N}$, and $G^{+}:L^{2}(D)^{d_{in}} \to L^{2}(D)^{d_{out}}$ be continuous with boundedness as in (\ref{bound-R-G-M}).
Assume that the non-linear activation function $\sigma$ is either ReLU or Leaky ReLU.
Then, for any compact set $K \subset L^{2}(D)^{d_{in}}$, $\epsilon \in (0,1)$, there exists $L \in \mathbb{N}$, $N^{\prime} \in \mathbb{N}$ with 
\begin{equation}
N^{\prime}d_{out} \geq 2Nd_{in} +1,
\label{N-N^prime}
\end{equation}
and $G_{N^{\prime}} \in \mathrm{NO}_{L,N^{\prime}}(\sigma; D, d_{in}, d_{out})$ such that 
$
G_{N^{\prime}}:(\mathrm{span}\{\varphi_{k}\}_{k \leq N})^{d_{in}} \to (\mathrm{span}\{\varphi_{k}\}_{k \leq N^{\prime}})^{d_{out}}
$
is injective and
\[
\sup_{a \in K} \left\|G^{+}(a) - G_{N^{\prime}}(a) \right\|_{L^{2}(D)^{d_{out}}} \leq \epsilon.
\]
\end{theorem}
See Section~\ref{proof-Universal-injectivity-NO-finite-rank} in Appendix~\ref{Appendix3} for the proof. 
In the proof, we make use of \citet[Lemma 29]{puthawala2022globally}, which gives rise to the assumption (\ref{N-N^prime}).
We do not require any condition on $d_{in}$ and $d_{out}$ as well as Theorem~\ref{Universal-injectivity-NO}.

\begin{remark}
    Observe that in our finite-rank approximation result, we only require that the target function $G^+$ is continuous and bounded, but not smooth. Further, our error is measured in the $\sup$ norm, which suggests that our approximation doesn't smoothen non-smooth operators. 
\end{remark}

\section{Subnetworks \& nonlinear integral operators: bijectivity and inversion
}\label{Non-linear operator}

So far our analysis of injectivity has been restricted to the case where the only source of nonlinearity are the activation functions. In this section we consider a weaker and more abstract problem where nonlinearities can also arise from the integral kernel with surjective activation function, such as leaky $\relu$. Specifically, we consider layers of the form
\begin{equation}
F_1(u)=Wu+K(u),
\label{definition-F_1}
\end{equation}
where $W \in \mathcal{L}(L^2(D)^{n}, L^2(D)^{n})$ is a linear bounded bijective operator, and $K:L^{2}(D)^n \to L^{2}(D)^n$ is a non-linear operator. This arises in the non-linear neural operator construction by \citet{kovachki2021neural} or in \citet{ong2022iae} to improve performance of integral autoencoders. In this construction, each layer $\mathcal{L}_{\ell}$ is written as 
\[
\quad x \in D, \quad(\mathcal{L}_{\ell}v)(x)=\sigma (W_{\ell}v(x) + K_{\ell}(v)(x) ), \quad K_{\ell}(u)(x)=\int_{D}k_{\ell}(x,y,u(x),u(y))u(y)dy,
\]
where $W_{\ell} \in \mathbb{R}^{d_{\ell+1} \times d_{\ell}}$ independent of $x$, and $K_{\ell}:L^{2}(D)^{d_{\ell}} \to L^{2}(D)^{d_{\ell+1}}$ is the non-linear integral operator.

This relaxing of assumptions is motivated by a desire to obtain theoretical results for both subnetworks and operator transformers. By subnetworks, we mean compositions of layers within a network. This includes, for example, the encoder or decoder block of a traditional VAE. By neural operator we mean operator generalizations of finite-rank transformers, which can be modeled by letting $K$ be an integral transform with nonlinear kernel $k$ of the form 
\begin{align*}
    k(x,y,v(x),v(y)) \equiv \mathrm{softmax} \circ \innerprod{Av(x)}{Bv(y)},
\end{align*}
where $A$ and $B$ are matrices of free parameters, and the (integral) softmax is taken over $x$. This specific choice of $k$ can be understood as a natural generalization of the attention mechanism in transformers, see \cite[Sec. 5.2]{kovachki2021neural} for further details.

\subsection{Surjectivity and bijectivity}\label{Surjectivity and bijectivity}
Recall from \citet[Sec 2, Chap VII]{showalter2010hilbert} that a non-linear operator $K:L^2(D)^{n} \to L^2(D)^{n}$  is coercive if
\begin{equation}
\lim_{\|u\|_{L^2(D)^n}\to \infty}{\bra K (u), \frac  u {\|u\|_{L^2(D)^n}} \ket_{L^2(D)^n}} =\infty.
\label{coercivity}
\end{equation}
\begin{proposition}
\label{lem:surjectivity-of-injective-nonlinear-operators}
Let $\sigma : \mathbb{R} \to \mathbb{R}$ be surjective and $W:L^{2}(D)^n \to L^{2}(D)^n$ be linear bounded bijective (then the inverse $W^{-1}$ is bounded linear), 
and let $K:L^{2}(D)^n \to L^{2}(D)^n$ be a continuous and compact mapping.
Moreover, assume that the map $u \mapsto \alpha u+W ^{-1}K(u)$ is coercive with some $0<\alpha<1$. 
Then, the operator $\sigma \circ F_{1}$ is surjective. 
\end{proposition}
See Section~\ref{proof-lem:surjectivity-of-injective-nonlinear-operators} in Appendix~\ref{Appendix4} for the proof.
Here, we provide the example for $K$ satisfying the coercivity (\ref{coercivity}).
\begin{example}
We simply consider the case of $n=1$, and $D \subset \mathbb{R}^{d}$ a is bounded interval.
We consider the non-linear integral operator 
\[
K(u)(x):=\int_{D} k(x,y,u(x)) u(y)dy, \ x \in D.
\]
The operator $u \mapsto \alpha u+W ^{-1}K(u)$ with some $0<\alpha<1$ is coercive when 
the non-linear integral kernel $k(x,y,t)$ satisfies certain boundedness conditions. In
Examples~\ref{example-layer-sur-app} and \ref{example-layer-sur-appB} in Appendix~\ref{Appendix4}, we show that these conditions are met by kernels $k(x,y,t)$ of the form
\[
k(x,y,t)=\sum_{j=1}^{J} c_{j}(x,y)\sigma(a_j(x,y)t+b_j(x,y)),
\]
where $\sigma:\R\to \R$ is the sigmoid function  $\sigma_s:\R\to \R$, and $a,b, c \in C(\overline D\times \overline D)$ or by a wavelet activation function $\sigma_{wire}:\R\to \R$, see \citet{saragadam2023wire}, where
$\sigma_{wire}(t)=\hbox{Im}\,(e^{i\omega t}e^{-t^2})$ and $a,b, c 
\in C(\overline D\times \overline D)$, and $a_j(x,y)\not =0$.
\end{example}

In the proof of Proposition~\ref{lem:surjectivity-of-injective-nonlinear-operators}, we utilize the Leray-Schauder fix point theorem.
By employing the Banach fix point theorem under a contraction mapping condition (\ref{contraction mapping}), we can obtain the bijectivity as follows:
\begin{proposition}\label{injecitivity-non-linear-NOs}
Let $\sigma : \mathbb{R} \to \mathbb{R}$ be bijective.
Let $W:L^{2}(D)^n \to L^{2}(D)^n$ be bounded linear bijective, and let $K:L^{2}(D)^n \to L^{2}(D)^n$.
If $W^{-1} K:L^{2}(D)^n \to L^{2}(D)^n$ 
is a contraction mapping, that is, there exists $\rho \in (0,1)$ such that
\begin{equation}
\left\|W^{-1} K(u) - W^{-1} K(v) \right\|
\leq \rho \left\|u - v \right\|, \ u,v \in L^2(D)^n,
\label{contraction mapping}
\end{equation}
then, the operator $\sigma \circ F_1$ is bijective. 
\end{proposition}

See Section~\ref{proof-injecitivity-non-linear-NOs} in Appendix~\ref{Appendix4} for the proof.
We note that if $K$ is compact linear, then assumption~(\ref{contraction mapping}) implies that $W+K$ is an injective Fredholm operator with index zero, which is equivalent to $\sigma \circ (W+K)$ being bijective as observed in Proposition~\ref{injective-activation-chara}.
That is, Proposition~\ref{injecitivity-non-linear-NOs} requires stronger assumptions when applied to the linear case.

Assumption~(\ref{contraction mapping}) implies the the injectivity of $K$.
An interesting example of injective operators arises when $K$ are Volterra operators.
When $D\subset \mathbb R^d$ is bounded and $K(u)=\int_D k(x,y,u(y))u(y)dy$, 
where we denote $x=(x_1,\dots,x_d)$ and $y=(y_1,\dots,y_d)$, we recall that $K$ is a Volterra operator 
if $k(x,y,t)\not =0$ implies 
$y_j\le x_j$ for all $j=1,2,\dots,d$.
A well known fact, as discussed in Example \ref{ex: Volterra} in Appendix~\ref{Appendix4}, is that if $K(u)$ is a Volterra operator 
whose kernel $k(x,y,t)\in C(\overline D\times \overline D\times \R^n)$ is bounded and uniformly Lipschitz in $t$-variable then $F:u\mapsto u+K(u)$ is injective.

\subsection{Construction of the inverse of a non-linear integral neural operator}\label{Construction of the inverse of a non-linear integral neural operator}

The preceding sections clarified sufficient conditions for surjectivity and bijectivity of the non-linear operator $F_1$.
We now consider how to construct of the inverse of $F_1$ in a compact set $\mathcal Y$. We find that constructing inverses is possible in a wide variety of settings and, moreover, that the inverses are themselves can be given by neural operators. We prove this in the rest of this section, but first we summarize the main three steps of the proof.
\begin{itemize}
    \item By using the Banach fixed point theorem and invertibility of derivatives of $F_1$ we show that, locally, $F_1$ may be inverted by an iteration of a contractive operator near $g_j = F_1(v_j)$. This makes local inversion simple in balls which cover the set $\mathcal Y$.
    \item Next, we construct partition of unity functions $\Phi_j$ that masks the support of each covering ball and allows us to construct one global inverse that simply passes through to the local inverse on the appropriate ball.
    \item Finally we show that each function used in both of the above steps are representable using neural operators with distributional kernels.
\end{itemize}

As simple case, let us first consider the case when $n=1$, and $D\subset \R$ is a bounded interval, and the operator $F_1$ of the form
\ba
F_1(u)(x)=W(x) u(x)+\int_Dk(x,y,u(y)) u(y)dy,
\ea
where $W\in C^1(\overline D)$ satisfies $0< c_1\le W(x)\le c_2$ and
the function $(x,y,s)\mapsto k(x,y,s)$
is in $C^3(\overline D \times \overline D \times \R)$ and 
in $\overline D \times \overline D \times \R$ its three derivatives and the derivatives of $W$ are uniformly bounded
by $c_0$, that is,
\beq\label{eq: kernel k-main}
\|k\|_{C^3(\overline D \times \overline D \times \R)}\leq c_0,\quad
\|W\|_{C^1(\overline D)}\leq c_0.
\eeq
The condition \eqref{eq: kernel k-main} implies that 
\beq\label{F1 operators-main}
 F_1:H^1(D)\to H^1(D),
\eeq
is locally Lipschitz smooth functions. 
Furthermore, $F_1 :H^1(D)\to H^1(D)$ is Fr\'{e}chet differentiable at $u_0\in C(\overline D)$, and we denote Fr\'{e}chet derivative of $F_1$ at $u_0$ by $A_{u_0}$, which can be written as the integral operator (\ref{A-operator}).
We will below assume that for all $u_0\in C(\overline D)$
the integral operator 
\beq\label{A is injective-main}
A_{u_0}:H^1(D)\to H^1(D)\hbox{  is an injective operator}.
\eeq
This happens for example when $K(u)$ is a  Volterra operator, see Examples \ref{ex: Volterra} and
\ref{ex: Volterra derivative}.
As the integral operators $A_{u_0}$ are Fredholm operators having
index zero, this implies that the operators \eqref{A is injective-main} are bijective.
The inverse operator $A_{u_0}^{-1}:H^1(D)\to H^1(D)$ can be written by the integral operator (\ref{inverse-A_u_0}).

We will consider the inverse function of the map $F_1$ in $\mathcal Y\subset \sigma_a(\overline B_{C^{1,\alpha}(\overline D)}(0,R))=\{\sigma_a\circ g\in C(\overline D):\
\|g\|_{C^{1,\alpha}(\overline D)}\leq R\}$, which is a set of the image of H\"older spaces $C^{n,\alpha}(\overline D)$ through (leaky) ReLU-type functions $\sigma_a(s)=\relu(s)-a\relu(-s)$ with $a\ge 0$.
We note that $\mathcal Y$  a compact subset the Sobolev space $H^1(D)$, and we use the notations $B_{\mathcal X}(g,R)$ and $\overline B_{\mathcal X}(g,R)$ as open and closed balls with radius $R>0$ at center $g \in \mathcal X$ in Banach space $\mathcal X$.

To this end, we will cover the set $\mathcal Y$ with small balls $B_{H^1(D)}(g_j,\e_0)$,
$j=1,2,\dots,J$ of $H^1(D)$, centered at 
 $g_j=F_1(v_j)$, where $v_j\in H^1(D)$.
As considered in detail in Appendix~\ref{Appendix5}, when $g$ is sufficiently near to the function $g_j$
in $H^1(D)$, the inverse map of $F_1$ can be written as a limit
$(F_1^{-1}(g),g)=\lim_{m\to \infty} \mathcal H_j^{\circ m}(v_j,g)$ in $H^1(D)^2$, 
where
  \ba
\mathcal H_j \left(\begin{array} {c}u \\ g \end{array}\right) 
\coloneqq
 \left(\begin{array} {c}
 u-A_{v_j}^{-1}(F_1(u)-F_1(v_j))+A_{v_j}^{-1}(g-g_j)\\ 
  g \end{array}\right),
 \ea
 that is, near $g_j$ we can approximate $F_1^{-1}$ as a composition $\mathcal H_j^{\circ m}$ of $2m$ layers of neural operators. 

To glue the local inverse maps together, we use a partition of unity $\Phi_{\vec i}$, ${\vec i}\in \mathcal{I}$ in the function space $\mathcal Y$, where $\mathcal{I}\subset \mathbb Z^{\ell_0}$ is a
finite index set. The function $\Phi_{\vec i}$
are given by neural operators
\ba
\Phi_{\vec i}(v,w)=\pi_1\circ \phi_{{\vec i},1}\circ \phi_{{\vec i},2}\circ \dots\circ \phi_{{\vec i},\ell_0}(v,w),\quad\hbox{where} \quad \phi_{{\vec i},\ell}(v,w)=(F_{y_\ell,s({\vec i},\ell),\epsilon_1}(v,w),w),
\ea
where some $\epsilon_1>0$, $s({\vec i},\ell)\in \mathbb R$ are some suitable values near $g_{j({\vec i})}(y_\ell)$, some $y_\ell \in D$ ($\ell=1,...,\ell_0$),
and $\pi_1$ is the map $\pi_1(v,w)=v$ that maps a pair $(v,w)$ to the first function $v$.
Here, $F_{z,s,h}(v,w)$ are integral neural operators with distributional kernels
\ba
F_{z,s,h}(v,w)(x)=\int_D k_{z,s,h}(x,y,v(x),w(y))dy,
\ea
where $
k_{z,s,h}(x,y,v(x),w(y))=v(x){\bf 1}_{[s-\frac 12h,s+\frac 12h)}(w(y))\delta(y-z)$, and ${\bf 1}_A$ is indicator function of a set $A$ and $y\mapsto\delta(y-z)$ is the Dirac
delta distribution at the point $z\in D$. Using these, we can write the inverse of $F_1$ at $g\in\mathcal Y$ as 
 \beq\label{inverse as a limit-main}
   F_1^{-1}(g)=\lim_{m\to \infty} \sum_{{\vec i}\in \mathcal I}\Phi_{{\vec i}}    
  \mathcal H_{j({\vec i})}^{\circ m} \left(\begin{array} {c}v_{j({\vec i})} \\ g \end{array}\right)\quad\hbox{in }H^1(D)
 \eeq
where ${j({\vec i})}\in \{1,2,\dots,J\}.$

This result is summarized in following theorem which is proved in Appendix~\ref{Appendix5}.
\begin{theorem}\label{thm: invertibility} 
Assume that $F_1$ satisfies the above assumptions  \eqref {eq: kernel k-main}
 and \eqref{A is injective-main} and that 
 $F_1:H^1(D)\to H^1(D)$ is a bijection.
Let $\mathcal Y\subset  \sigma_a(\overline B_{C^{1,\alpha}(\overline D)}(0,R))$ be a compact subset the Sobolev space $H^1(D)$,
where $\alpha>0$ and $a\ge 0$.
Then the inverse of $F_1:H^1(D)\to H^1(D)$ in $\mathcal Y$ can 
written as a limit \eqref{inverse as a limit-main} that is, as a limit
of integral neural operators.
\end{theorem}

\section{Discussion}
\label{Discussion}

In recent years, generative models in infinite-dimensional function space \citep{burt2020understanding, rudner2022continual, dupont2021generative, lim2023score, ong2022iae, alberti2022continuous, lim2023score,burt2020understanding, rudner2022continual, ong2022iae} have gained increasing interest.
These are useful for function-based modeling, where various physical quantities are modeled by functions. 
Notably, \citet{ong2022iae} introduced the integral autoencoder, which achieves discretization invariant learning.
Their architecture utilizes encoders and decoders represented by integral operators, closely resembling the framework of neural operators.
By combining their work with our injective work, we expect the development of injective generator that have desirable proprieties of invariant discretization. 

Inverse problems in infinite dimensional spaces, especially based on PDE models, have also recent advancements in data-driven approaches \citep{arridge2019solving}. 
In the analysis of inverse problems, the uniqueness, whether the forward operator is injective, is a key concern \citep{isakov2006inverse}, and in their applications, it is important to construct the injective forward operator and its inverse operator when dealing with models that hold uniqueness.
In this context, injective neural operators can serve as surrogate forward and inverse operators.

Therefore, it is crucial to provide the exact algorithm for constructing injective neural operators and their inverses, to facilitate their application in these areas. 
This will be developed in the future. 

\section{Conclusion}

In this paper, we provided a theoretical analysis of injectivity and bijectivity for neural operators.
In the future, we will further develop applications of our theory, particularly in the areas of generative models and inverse problems and integral autoencoders. We gave a rigorous framework for the analysis of the injectivity and bijectivity of neural operators including when either the $\relu$ activation or bijective activation functions are used. We further proves that injective neural operators are universal approximators and their finite-rank implementation are still injective. Finally, we ended by consider the `coarser' problem of non-linear integral operators, as arises in subnetworks, operator transformers and integral autoencoders.

\section*{Acknowledgments}

MP was partially supported by CAPITAL Services of Sioux Falls, South Dakota. ML was partially supported by Academy of Finland, grants 273979, 284715, 312110.
M.V.\ de~H. was supported by the Simons Foundation under the MATH + X program, the National Science Foundation under grant DMS-2108175, and the corporate members of the Geo-Mathematical Imaging Group at Rice University.

\bibliography{ref.bib} 
\bibliographystyle{plainnat}

\newpage
\appendix
\part*{Appendix}

\section{Proof of Proposition~\ref{injectivity-ReLU:general} in Section~\ref{Layerwise-injectivity}}
\label{Appendix1}
\label{proof-injectivity-ReLU:general}

\begin{proof}

We use the notation  $T|_{S(v,T+b)}(v)=(T_iv)_{i\in S(v,T+b)}$.
Assume that $T+b$ has a DSS with respect to every $v \in L^{2}(D)^{n}$ in the sense of Definition~\ref{definition:DSS:general},
and that
\begin{equation}
\relu(Tv^{(1)}+b)
=
\relu(Tv^{(2)}+b) \quad \text{ in } D, \label{Assumption:DSS:general}
\end{equation}
where $v^{(1)}, v^{(2)} \in L^{2}(D)^{n}$.
Since $T+b$ has a DSS with respect to $v^{(1)}$, we have for $i \in S(v^{(1)}, T+b)$
\[
0< \relu(T_{i}v^{(1)}+b_{i})
=
\relu(T_{i}v^{(2)}+b_{i}) \text{ in } D,
\]
which implies that 
\[
T_{i}v^{(1)} + b_{i} =T_{i}v^{(2)} + b_{i} \text{ in } D.
\]
Thus, 
\begin{equation}
v^{(1)}-v^{(2)} \in \mathrm{Ker}\left(T\bigl|_{S(v^{(1)},T+b)}\right). \label{step1-Proof-injectivity-ReLU:general}
\end{equation}
By assuming (\ref{Assumption:DSS:general}), we have for $i \notin S(v^{(1)}, T)$,
\[
\{x \in D \ | \ T_{i}v^{(1)}(x)+ b_{i}(x) \leq 0 \}
=
\{x \in D \ | \ T_{i}v^{(2)}(x)+ b_{i}(x) \leq 0 \}.
\]
Then, we have
\[
T_{i}(v^{(1)} - ( v^{(1)} - v^{(2)} ) )(x) + b_{i}(x) = T_{i} v^{(2)}(x) + b_{i}(x) \leq 0 \text{ if } T_{i}v^{(1)}(x) + b_{i}(x)\leq 0,
\]
that is, 
\[
T_{i}v^{(1)}(x) + b_{i}(x) \leq T_{i}\left( v^{(1)} - v^{(2)} \right)(x) \text{ if } T_{i}v^{(1)}(x) + b_{i}(x)\leq 0.
\]
In addition,
\[
T_{i}(v^{(1)} - v^{(2)})(x)=T_{i}v^{(1)}(x)+b_{i}(x) - \left( Tv^{(2)}(x)+b_{i}(x) \right) = 0 \text{ if } T_{i}v^{(1)}(x)+ b_{i}(x) >0.
\]
Thus, 
\begin{equation}
v^{(1)}-v^{(2)} \in X(v,T+b). \label{step2-Proof-injectivity-ReLU:general}
\end{equation}
Combining (\ref{step1-Proof-injectivity-ReLU:general}) and (\ref{step2-Proof-injectivity-ReLU:general}), and (\ref{zero-set-DSS}) as $v=v^{(1)}$, we conclude that 
$$
v^{(1)}-v^{(2)}=0.
$$

\par
Conversely, assume that there exists a $v \in L^{2}(D)^{n}$ such that 
\[
\mathrm{Ker}\left(T\bigl|_{S(v,T+b)}\right) \cap X(v, T+b) \neq \{  0\}. 
\]
Then there is $ u \neq 0$ such that
\[
u \in \mathrm{Ker}\left(T\bigl|_{S(v,T+b)}\right) \cap X(v, T+b).
\]
For $i \in S(v, T+b)$, we have by $u \in \mathrm{Ker}(T_i)$,
\[
\relu\left(T_{i}(v-u)+b_{i}(x)\right)
=
\relu\left(T_{i}v+b_{i}(x) \right).
\]
For $i \notin S(v, T+b)$, we have by $u \in X(v,T+b)$,
\[
\begin{split}
\relu\left(T_{i}(v-u)(x)+b_{i}(x) \right)
&=
\left\{
\begin{array}{l}
0 \quad \text{if } T_iv(x)+b_{i}(x) \leq 0  \\
T_iv(x)+b_{i}(x) \quad \text{if } T_iv(x)+b_{i}(x)>0 
\end{array}
\right.
\\
&
=
\relu\left(T_{i}v(x)+b_{i}(x) \right).
\end{split}
\]
Therefore, we conclude that
\[
\relu\left(T(v-u)+b\right)
=
\relu\left(Tv +b \right),
\]
where $u \neq 0$, that is, $\relu\circ (T+b)$ is not injective.
\end{proof}

\section{Details of Sections~\ref{Global analysis} and \ref{Universal approximation}}\label{Appendix2}
\subsection{Proof of Lemma~\ref{injective-dimensional-reduction}}
\label{proof-injective-dimensional-reduction}
\begin{proof}
The restriction operator, $\pi_{\ell}: L^{2}(D)^{m} \to L^{2}(D)^{\ell}$ ($\ell < m$), acts as follows,
\begin{equation}
\pi_{\ell}(a, b) := b, \quad (a,b) \in L^{2}(D)^{m-\ell} \times L^{2}(D)^{\ell}.
\label{restriction-pi}
\end{equation}

Since $L^{2}(D)$ is a separable Hilbert space, there exists an orthonormal basis $\{ \varphi_{k} \}_{k \in \mathbb{N}}$ in $L^{2}(D)$.
We denote by
\[
\varphi^{0}_{k,j}
:=
\Biggl( 
0,...,0, \underbrace{\varphi_{k}}_{j-th}, 0,...,0 
\Biggr) \in L^{2}(D)^{m},
\]
for $k \in \mathbb{N}$ and $j \in [m-\ell]$.
Then, $\{ \varphi^{0}_{k,j}\}_{k \in \mathbb{N}, j \in [m-\ell]}$ is an orthonormal sequence in $L^{2}(D)^{m}$, and 
\[
\begin{split}
V_{0} &
:= L^{2}(D)^{m-\ell} \times \{0\}^{\ell}
\\
&
= \mathrm{span} \left\{\varphi^{0}_{k,j} \ \bigl| \ k \in \mathbb{N}, \  j \in [m-\ell] \right\}.
\end{split}
\]
We define, for $\alpha \in (0, 1)$,
\begin{equation}\label{definition-varphi-alpha}
\varphi^{\alpha}_{k,j}
:=
\Biggl( 
0,...,0, \underbrace{\sqrt{(1-\alpha)} \varphi_{k}}_{j-th}, 0,...,0, \sqrt{\alpha} \xi_{(k-1)(m-\ell)+j} 
\Biggr) \in L^{2}(D)^{m},
\end{equation}
with $k \in \mathbb{N}$ and $j \in [m-\ell]$.
We note that $\{ \varphi^{\alpha}_{k,j}\}_{k \in \mathbb{N}, j \in [m-\ell]}$ is an orthonormal sequence in $L^{2}(D)^{m}$.
We set
\begin{equation}
V_{\alpha}
:= \mathrm{span}\left\{\varphi^{\alpha}_{k,j} \ \Bigl| \ k \in \mathbb{N}, j \in [m-\ell] \right\}.
\label{def-V-alpha-K}
\end{equation}
It holds for $0<\alpha<1/2$ that
\[
\left\| P_{V_{\alpha}^{\perp}} - P_{V_{0}^{\perp}} \right\|_{\mathrm{op}} < 1.
\]
Indeed, for $u \in L^2(D)^{m}$ and $0<\alpha<1/2$, 

\[
\begin{split}
&
\left\| P_{V_{\alpha}^{\perp}}u - P_{V_{0}^{\perp}}u \right\|_{L^2(D)^{m}}^{2}
=
\left\| P_{V_{\alpha}} u - P_{V_{0}} u \right\|_{L^2(D)^{m}}^{2}
\\
&
=
\left\| \sum_{k \in \mathbb{N},j \in [m-\ell] }(u, \varphi_{k,j}^{\alpha} )\varphi_{k,j}^{\alpha} - (u, \varphi_{k,j}^{0} )\varphi_{k,j}^{0} \right\|_{L^2(D)^{m}}^{2}
\\
&
=
\left\| \sum_{k \in \mathbb{N},j \in [m-\ell] } (1-\alpha) (u_{j}, \varphi_{k} )\varphi_{k} - (u_{j}, \varphi_{k} )\varphi_{k} \right\|_{L^2(D)}^{2}
\\
&
+
\left\| \sum_{k \in \mathbb{N},j \in [m-\ell]} \alpha (u_{m}, \xi_{(k-1)(m-\ell)+j} )\xi_{(k-1)(m-\ell)+j} \right\|_{L^2(D)}^{2}
\\
&
\leq
\alpha^2\sum_{j \in [m-\ell]} \sum_{k \in \mathbb{N}} |(u_{j}, \varphi_{k})|^{2} + 
\alpha^{2} \sum_{k \in \mathbb{N}} |(u_{m}, \xi_{k})|^{2} 
\leq 4 \alpha^2 \left\| u \right\|_{L^2(D)^{m}}^2,
\end{split}
\]
which implies that $\left\| P_{V_{\alpha}^{\perp}} - P_{V_{0}^{\perp}} \right\|_{\mathrm{op}} \leq 2 \alpha$. 
\par
We will show that the operator
\[
P_{V_{\alpha}^{\perp}} \circ T : L^{2}(D)^{n} \to L^{2}(D)^{m},
\]
is injective. Assuming that for $a, b \in L^2(D)^{n}$,
\[
P_{V_{\alpha}^{\perp}} \circ T(a)=P_{V_{\alpha}^{\perp}} \circ T(b),
\]
is equivalent to
\[
T(a) - T(b) =
P_{V_{\alpha}}(T(a) - T(b)).
\]
Denoting by 
$P_{V_{\alpha}}(T(a) - T(b))=\sum_{k \in \mathbb{N},j \in [m-\ell]}c_{k,j}\varphi_{k,j}^{\alpha}$,
\[
\pi_{1}(T(a) - T(b))
=
\sum_{k \in \mathbb{N},j \in [m-\ell]}c_{k,j}\xi_{(k-1)(m-\ell)+j}.
\]
From (\ref{ONS-notin-range-main}), we obtain that $c_{kj}=0$ for all $k,j$.
By injectivity of $T$, we finally get $a=b$.

We define $Q_{\alpha} : L^{2}(D)^{m} \to L^{2}(D)^{m}$ by
\[
Q_{\alpha} :=\left( P_{V_{0}^{\perp}} P_{V_{\alpha}^{\perp}} + (I-P_{V_{0}^{\perp}}) (I-P_{V_{\alpha}^{\perp}}) \right)
\left( I-(P_{V_{0}^{\perp}}-P_{V_{\alpha}^{\perp}})^2 \right)^{-1/2}.
\]
By the same argument as in Section~I.4.6 \citet{kato2013perturbation}, we can show that $Q_{\alpha}$ is injective and 
\[
Q_{\alpha} P_{V_{\alpha}^{\perp}} = P_{V_{0}^{\perp}} Q_{\alpha},
\]
that is, $Q_{\alpha}$ maps from $\mathrm{Ran}(P_{V_{\alpha}^{\perp}})$ to 
\[
\mathrm{Ran}(P_{V_{0}^{\perp}}) \subset \{0\}^{m-\ell} \times L^{2}(D)^{\ell}.
\]
It follows that
\[
\pi_{\ell} \circ Q_{\alpha} \circ P_{V_{\alpha}^{\perp}} \circ T : L^{2}(D)^{n} \to L^{2}(D)^{\ell}
\]
is injective.
\end{proof}

\subsection{Remarks following Lemma~\ref{injective-dimensional-reduction}}
\begin{remark}\label{remark-choice-xi}
An example that satisfies (\ref{ONS-notin-range-main}) is the neural operator whose $L$-th layer operator $\mathcal{L}_{L}$ consists of the integral operator $K_{L}$ with continuous kernel function $k_{L}$, and with continuous activation function $\sigma$.
Indeed, in this case, we may choose the orthogonal sequence $\{ \xi_{k} \}_{k \in \mathbb{N}}$ in $L^{2}(D)$ as a discontinuous functions sequence \footnote{e.g., step functions whose supports are disjoint for each sequence.} so that $\mathrm{span}\{\xi_{k}\}_{k \in \mathbb{N}} \cap C(D) = \{ 0 \}$. Then, by $\mathrm{Ran}(\mathcal{L}_L) \subset C(D)^{d_{L}}$, the assumption~(\ref{ONS-notin-range-main}) holds.
\label{remark-conti-case-B}
\end{remark}

\begin{remark}\label{existential-B}
In the proof of Lemma~\ref{injective-dimensional-reduction}, an operator $B \in \mathcal{L}(L^{2}(D)^{m}, L^{2}(D)^{\ell})$,
\[
B
=
\pi_{\ell} \circ Q_{\alpha} \circ P_{V_{\alpha}^{\perp}},
\]
appears, where $P_{V_{\alpha}^{\perp}}$ is the orthogonal projection onto orthogonal complement $V_{\alpha}^{\perp}$ of $V_{\alpha}$ with
\[
V_{\alpha}
:= \mathrm{span}\left\{\varphi^{\alpha}_{k,j} \ \Bigl| \ k \in \mathbb{N}, j \in [m-\ell] \right\}
\subset
L^{2}(D)^{m},
\] 
in which $\varphi^{\alpha}_{k,j}$ is defined for $\alpha \in (0, 1)$, $k \in \mathbb{N}$ and $j \in [\ell]$,
\[
\varphi^{\alpha}_{k,j}
:=
\Biggl( 
0,...,0, \underbrace{\sqrt{(1-\alpha)}\varphi_{k}}_{j-th}, 0,...,0, \sqrt{\alpha} \xi_{(k-1)(m-\ell)+j} 
\Biggr).
\]
Here, $\{ \varphi_{k} \}_{k \in \mathbb{N}}$ is an orthonormal basis in $L^{2}(D)$.
Futhermore, $Q_{\alpha} : L^{2}(D)^{m} \to L^{2}(D)^{m}$ is defined by
\[
Q_{\alpha} :=\left( P_{V_{0}^{\perp}} P_{V_{\alpha}^{\perp}} + (I-P_{V_{0}^{\perp}}) (I-P_{V_{\alpha}^{\perp}}) \right)
\left( I-(P_{V_{0}^{\perp}}-P_{V_{\alpha}^{\perp}})^2 \right)^{-1/2},
\]
where $P_{V_{0}^{\perp}}$ is the orthogonal projection onto orthogonal complement $V_{0}^{\perp}$ of $V_{0}$ with
\[
V_{0} 
:= L^{2}(D)^{m-\ell} \times \{0\}^{\ell}.
\]
The operator $Q_{\alpha}$ is well-defined for $0<\alpha<1/2$ because it holds that
\[
\left\| P_{V_{\alpha}^{\perp}} - P_{V_{0}^{\perp}} \right\|_{\mathrm{op}} < 2 \alpha.
\]

\medskip

This construction is given by the combination of "Pairs of projections" discussed in \citet[Section~I.4.6]{kato2013perturbation} with the idea presented in \cite[Lemma 29]{puthawala2022universal}.
\end{remark}

\medskip

\subsection{Proof of Theorem~\ref{Universal-injectivity-NO}}
\label{proof-Universal-injectivity-NO}

We now prove Theorem~\ref{Universal-injectivity-NO}.

\begin{proof}
Let $R>0$ such that
\[
K \subsetneq B_{R}(0),
\]
where $B_{R}(0):=\{ u \in L^{2}(D)^{d_{in}} \ | \ \left\| u \right\|_{L^{2}(D)^{d_{in}}} \leq R \}$. 
By Theorem 11 of \citet{kovachki2021neural}, there exists $L \in \mathbb{N}$ and $\widetilde{G} \in \mathrm{NO}_{L}(\sigma ; D, d_{in}, d_{out})$ such that
\begin{equation}
\sup_{a \in K} \left\| G^{+}(a) - \widetilde{G}(a) \right\|_{L^{2}(D)^{d_{out}}} \leq \frac{\epsilon}{2},
\label{from-Thm11-kova}
\end{equation}
and
\[
\left\| \widetilde{G}(a) \right\|_{L^{2}(D)^{d_{out}}} \leq 4 M, \quad\hbox{for } a \in L^2(D)^{d_{in}},\quad
\|a\|_{L^2(D)^{d_{in}}}\le R.
\]
We write operator $\widetilde{G}$ by
\[
\widetilde{G} = \widetilde{K}_{L+1} \circ (\widetilde{K}_{L}+\widetilde{b}_{L}) \circ \sigma \cdots  \circ (\widetilde{K}_{2}+\widetilde{b}_{2}) \circ \sigma \circ (\widetilde{K}_{1}+\widetilde{b}_{1}) \circ (\widetilde{K}_{0}+\widetilde{b}_{0}),
\]
where
\[
\begin{split}
&\widetilde{K}_{\ell} \in \mathcal{L}(L^{2}(D)^{d_{\ell}}, L^{2}(D)^{d_{\ell+1}}), \
\widetilde{K}_{\ell}: f \mapsto \int_{D}\widetilde{k}_{\ell}(\cdot, y)f(y)dy, 
\\
&
\widetilde{k}_{\ell} \in C(D \times D; \mathbb{R}^{d_{\ell+1} \times d_{\ell}}), \ 
\widetilde{b}_{\ell} \in L^{2}(D; \mathbb{R}^{d_{\ell+1} }),
\\[0.25cm]
&
d_{\ell} \in \mathbb{N}, \ 
d_{0}=d_{in}, \ d_{L+2}=d_{out}, 
\
\ell=0,...,L+2.
\end{split}
\]
We remark that kernel functions $\widetilde{k}_{\ell}$ are continuous because neural operators defined in \citet{kovachki2021neural} parameterize the integral kernel function by neural networks, thus,
\begin{equation}
\mathrm{Ran}(\widetilde{G}) \subset C(D)^{d_{out}}.
\label{range-G-tilde}
\end{equation}

We define the neural operator $H :L^{2}(D)^{d_{in}} \to L^{2}(D)^{d_{in}+d_{out}}$ by
\[
H = K_{L+1} \circ (K_{L}+b_{L}) \circ \sigma
\cdots \circ (K_{2}+b_{2}) \circ \sigma \circ (K_{1}+b_{1}) \circ (K_{0}+b_{0}),
\]
where $K_{\ell}$ and $b_{\ell}$ are defined as follows.
First, we choose $K_{inj} \in \mathcal{L}(L^{2}(D)^{d_{in}}, L^{2}(D)^{d_{in}})$ as 
a linear injective integral operator \footnote{For example, if we choose the integral kernel $k_{inj}$ as $k_{inj}(x,y)=\sum_{k=1}^{\infty}\vec{\varphi}_{k}(x)\vec{\varphi}_{k}(y)$, then the integral operator $K_{\mathrm{inj}}$ with the kernel $k_{\mathrm{inj}}$ is injective where $\{\vec{\varphi}\}_{k}$ is the orthonormal basis in $L^{2}(D)^{d_{in}}$.}. 
\vspace{2mm}
\\
(i) When $\sigma_1 \in \mathrm{A}_{0}^{L} \cap \mathrm{BA}$ is injective,

\[
K_{0}=
\left(
\begin{array}{c}
   K_{inj} \\
   \widetilde{K}_{0} \\ 
\end{array}
\right) \in \mathcal{L}(L^{2}(D)^{d_{in}}, L^{2}(D)^{d_{in}+d_{1}}), 
\quad
b_{0}=
\left(
\begin{array}{c}
   {\huge O} \\
   \widetilde{b}_{0} 
\end{array}
\right) \in L^{2}(D)^{d_{in}+d_{1}}, 
\]

\[
\vdots
\]
\[
K_{\ell}=
\left(
\begin{array}{cc}
K_{inj} & {\huge O} \\ 
{\huge O} &  \widetilde{K}_{\ell} \\
\end{array}
\right) \in \mathcal{L}(L^{2}(D)^{d_{in}+d_{\ell}}, L^{2}(D)^{d_{in}+d_{\ell+1}}), 
\quad
b_{\ell}=
\left(
\begin{array}{c}
{\huge O} \\ 
\widetilde{b}_{\ell} 
\end{array}
\right) \in L^{2}(D)^{d_{in}+d_{\ell+1}}, 
\]
\[
(1 \leq \ell \leq L),
\]

\[
\vdots
\]

\[
K_{L+1}=
\left(
\begin{array}{cc}
K_{inj} & {\huge O} \\ 
{\huge O} &  \widetilde{K}_{L+1} \\
\end{array}
\right) \in \mathcal{L}(L^{2}(D)^{d_{in}+d_{L+1}}, L^{2}(D)^{d_{in}+d_{out}}), 
\quad
b_{\ell}=
\left(
\begin{array}{c}
{\huge O} \\ 
{\huge O} \\ 
\end{array}
\right) \in L^{2}(D)^{d_{in} + d_{out}}.
\]
\\
(ii) When $\sigma_1 = \relu$,
\[
K_{0}=
\left(
\begin{array}{c}
   K_{inj} \\
   \widetilde{K}_{0} \\ 
\end{array}
\right) \in \mathcal{L}(L^{2}(D)^{d_{in}}, L^{2}(D)^{d_{in}+d_{1}}),
\quad
b_{0}=
\left(
\begin{array}{c}
   {\huge O} \\
   \widetilde{b}_{0} \\
\end{array}
\right) \in L^{2}(D)^{d_{in}+d_1}, 
\]
\[
K_{1}=
\left(
\begin{array}{cc}
K_{inj} & {\huge O} \\ 
-K_{inj}  & {\huge O} \\  
{\huge O} &  \widetilde{K}_{1}
\\ 
\end{array}
\right) \in \mathcal{L}(L^{2}(D)^{d_{in}+d_{1}}, L^{2}(D)^{2d_{in}+d_{2}}),
\
b_{0}=
\left(
\begin{array}{c}
   {\huge O} \\ 
   {\huge O} \\
   \widetilde{b}_{1} \\
\end{array}
\right) \in L^{2}(D)^{2 d_{in}+d_1}, 
\]
\[
\vdots
\]
\[
K_{\ell}=
\left(
\begin{array}{c:c}
\begin{array}{cc}
K_{inj} & -K_{inj} \\ 
-K_{inj} & K_{inj} \\  
\end{array}
 & {\huge O} \\ \hdashline
{\huge O} &  \widetilde{K}_{\ell}
\\ 
\end{array}
\right) \in \mathcal{L}(L^{2}(D)^{2d_{in}+d_{\ell}}, L^{2}(D)^{2d_{in}+d_{\ell+1}}),
\]
\[
b_{\ell}=
\left(
\begin{array}{c}
{\huge O} \\
{\huge O} \\
\widetilde{b}_{\ell} \\ 
\end{array}
\right) \in L^{2}(D)^{2 d_{in}+d_{\ell+1}}, 
\quad
(2\leq \ell \leq L),
\]
\[
\vdots
\]
\[
K_{L}=
\left(
\begin{array}{c:c}
\begin{array}{cc}
K_{inj} & -K_{inj} \\ 
\end{array}
 & {\huge O} \\ \hdashline
{\huge O} &  
\widetilde{K}_{L}
\\ 
\end{array}
\right) \in \mathcal{L}(L^{2}(D)^{2d_{in}+d_{L}}, L^{2}(D)^{d_{in}+d_{L+1}}),
\]
\[
b_{L}=
\left(
\begin{array}{c}
{\huge O} \\
\widetilde{b}_{L} \\
\end{array}
\right) \in L^{2}(D)^{d_{in}+d_{L+1}}, 
\]
\[
K_{L+1}=
\left(
\begin{array}{cc}
K_{inj}  & {\huge O} \\ 
{\huge O} & \widetilde{K}_{L+1} \\
\end{array}
\right) \in \mathcal{L}(L^{2}(D)^{d_{in}+d_{L+1}}, L^{2}(D)^{d_{in}+d_{out}}),
\]
\[
b_{L+1}=
\left(
\begin{array}{c}
{\huge O} \\
{\huge O} \\
\end{array}
\right) \in L^{2}(D)^{d_{in}+d_{out}}.
\]

Then, the operator $H :L^{2}(D)^{d_{in}} \to L^{2}(D)^{d_{in}+d_{out}}$ has the form of
\[
H:=
\left\{
\begin{array}{l}
\left(
\begin{array}{c}
K_{inj} \circ K_{inj} \circ \sigma \circ K_{inj} \circ
\cdots \circ \sigma \circ
K_{inj}\circ K_{inj} 
\\
\widetilde{G} \\
\end{array}
\right)
\quad \text{in the case of (i)}.
\\
\left(
\begin{array}{c}
K_{inj} \circ \cdots \circ K_{inj} \\ 
\widetilde{G} \\
\end{array}
\right)
\quad \text{ in the case of (ii)}.
\end{array}
\right.
\]
For the case of (ii), we have used the fact
\[
\left(
\begin{array}{cc}
I & -I
\end{array}
\right)
\circ 
\relu
\circ
\left(
\begin{array}{c}
I \\
-I \\
\end{array}
\right)
= I.
\]
Thus, in both cases, $H$ is injective.

\medskip

In the case of (i), as $\sigma \in A_{0}^{L}$, we obtain the estimate
\[
\left\| \sigma (f) \right\|_{L^{2}(D)^{d_{in}}}
\leq \sqrt{2|D|d_{in}}C_{0} + \left\| f \right\|_{L^{2}(D)^{d_{in}}}, \ f \in L^{2}(D)^{d_{in}},
\]
where 
\[
C_{0}:=\sup_{x \in \mathbb{R}} \frac{|\sigma(x)|}{1+|x|} < \infty.
\]
Then we evaluate for $a \in K (\subset B_{R}(0))$,
\begin{equation}
\begin{split}
& \left\| H(a) \right\|_{L^{2}(D)^{d_{in}+d_{out}}}
\\
&
\leq
\left\| \widetilde{G}(a) \right\|_{L^{2}(D)^{d_{out}}}
+
\left\| K_{inj} \circ  K_{inj} \circ \sigma \circ K_{inj} \circ
\cdots \circ \sigma \circ
K_{inj}\circ K_{inj}(a) \right\|_{L^{2}(D)^{d_{in}}}
\\
&
\leq
4M + \sqrt{2|D|d_{in}}C_{0} \sum_{\ell=1}^{L}\left\|K_{inj} \right\|_{\mathrm{op}}^{\ell+1} +\left\|K_{inj} \right\|_{\mathrm{op}}^{L+2}R=:C_{H}.
\end{split}
\label{estimate-H-1}
\end{equation}
In the case of (ii), we find the estimate, for $a \in K$,
\begin{equation}
\left\| H(a) \right\|_{L^{2}(D)^{d_{in}+d_{out}}}
\leq
4M + \left\|K_{inj} \right\|_{\mathrm{op}}^{L+2}R < C_{H}.
\label{estimate-H-2}
\end{equation}
From (\ref{range-G-tilde}) (especially, $\mathrm{Ran}(\pi_{1}H) \subset C(D)$) and Remark~\ref{remark-conti-case-B}, we can choose an orthogonal sequence $\{ \xi_{k} \}_{k \in \mathbb{N}}$ in $L^{2}(D)$ such that (\ref{ONS-notin-range-main}) holds.
By applying Lemma~\ref{injective-dimensional-reduction}, as $T=H$, $n=d_{in}$, $m=d_{in}+d_{out}$, $\ell=d_{out}$, we find that
\[
G:= \underbrace{\pi_{d_{out}} \circ Q_{\alpha} \circ P_{V_{\alpha}^{\perp}}}
_{=:B}
\circ H : L^{2}(D)^{d_{in}} \to L^{2}(D)^{d_{out}},
\]
is injective. Here, $P_{V_{\alpha}^{\perp}}$ and $Q_{\alpha}$ are defined as in Remark~\ref{existential-B}; we choose $0<\alpha<<1$ such that
\[
\left\| P_{V_{\alpha}^{\perp}} - P_{V_{0}^{\perp}} \right\|_{\mathrm{op}} < \min\left( \frac{\epsilon}{10C_{H}}, 1 \right) =: \epsilon_{0},
\]
where $P_{V_{0}^{\perp}}$ is the orthogonal projection onto 
\[
V_{0}^{\perp}
:=  \{0\}^{d_{in}} \times L^{2}(D)^{d_{out}}.
\]
By the same argument as in the proof of Theorem 15 in \citet{puthawala2022globally}, we can show that
\[
\left\|I-Q_{\alpha} \right\|_{\mathrm{op}} \leq 4 \epsilon_{0}.
\]
Furthermore, since $B$ is a linear operator, $B \circ K_{L+1}$ is also a linear operator with integral kernel $\left(Bk_{L+1}(\cdot, y)\right)(x)$, where $k_{L+1}(x,y)$ is the kernel of $K_{L+1}$. This implies that
\[
G \in \mathrm{NO}_{L}(\sigma;  D, d_{in}, d_{out}).
\]
We get, for $a \in K$,
\begin{equation}
\left\|G^{+}(a) - G(a) \right\|_{L^{2}(D)^{d_{out}}}
\leq 
\underbrace{\left\| G^{+}(a) - \widetilde G(a) \right\|_{L^{2}(D)^{d_{out}}}}_{
\text{(\ref{from-Thm11-kova})}
\leq \frac{\epsilon}{2}}
+
\left\|\widetilde G(a) -  G(a) \right\|_{L^{2}(D)^{d_{out}}}.
\label{estimate-G^+-G-1}
\end{equation}
Using (\ref{estimate-H-1}) and (\ref{estimate-H-2}), we then obtain
\begin{equation}
\begin{split}
&\left\|\widetilde G(a) -  G(a) \right\|_{L^{2}(D)^{d_{out}}}
= \left\| \pi_{d_{out}} \circ H(a) - \pi_{d_{out}} \circ Q_{\alpha} \circ P_{V_{\alpha}^{\perp}} \circ H(a) \right\|_{L^{2}(D)^{d_{out}}}
\\
&
\leq 
\left\| \pi_{d_{out}} \circ (P_{V_{0}^{\perp}}-P_{V_{\alpha}^{\perp}}+P_{V_{\alpha}^{\perp}}) \circ H(a) - \pi_{d_{out}} \circ Q_{\alpha} \circ P_{V_{\alpha}^{\perp}} \circ H(a) \right\|_{L^{2}(D)^{d_{out}}}
\\
&
\leq 
\left\| \pi_{d_{out}} \circ (P_{V_{0}^{\perp}}-P_{V_{\alpha}^{\perp}}) \circ H(a) \right\|_{L^{2}(D)^{d_{out}}} +
\left\|\pi_{d_{out}} \circ (I - Q_{\alpha}) \circ P_{V_{\alpha}^{\perp}} \circ H(a) \right\|_{L^{2}(D)^{d_{out}}}
\\
&
\leq 5 \epsilon_{0} \left\| H(a) \right\|_{L^{2}(D)^{d_{in}+d_{out}}} 
\leq \frac{\epsilon}{2}.
\end{split}
\label{estimate-G^+-G-2}
\end{equation}
Combining (\ref{estimate-G^+-G-1}) and (\ref{estimate-G^+-G-2}), we conclude that
\[
\sup_{a \in K} \left\| 
G^{+}(a) - G(a) \right\|_{L^{2}(D)^{d_{out}}} \leq \frac{\epsilon}{2} + \frac{\epsilon}{2}= \epsilon.
\]
\end{proof}

\subsection{Remark following Theorem~\ref{Universal-injectivity-NO}}
\begin{remark}
\label{remark-L2-universality}
We make the following observations using Theorem~\ref{Universal-injectivity-NO}:
\begin{itemize}
\item[(i)] 
ReLU and Leaky ReLU functions belong to $\mathrm{A}_{0}^{L} \cap \mathrm{BA}$ due to the fact that $\{ \sigma \in C(\mathbb{R}) \ | \ \sigma \text{ is not a polynomial} \} \subseteq A_{0}$ (see \citet{pinkus1999approximation}), and both the ReLU and Leaky ReLU functions belong to $\mathrm{BA}$ (see Lemma C.2 in \citet{lanthaler2022error}).
We note that Lemma C.2 in \citet{lanthaler2022error} solely established the case for ReLU. 
However, it holds true for Leaky ReLU as well since the proof relies on the fact that the function $x \mapsto \min(\max(x, R), R)$ can be exactly represented by a two-layer ReLU neural network, and a two-layer Leaky ReLU neural network can also represent this function.
Consequently, Leaky ReLU is one of example that satisfies (ii) in Theorem~\ref{Universal-injectivity-NO}.

\item[(ii)] 
We emphasize that our infinite-dimensional result, Theorem~\ref{Universal-injectivity-NO}, deviates from the finite-dimensional result.
\citet[Theorem 15]{puthawala2022globally} assumes that $2d_{in}+1 \leq d_{out}$ due to the use of Whitney's theorem.
In contrast, Theorem~\ref{Universal-injectivity-NO} does not assume any conditions on $d_{in}$ and $d_{out}$, that is, we are able to avoid invoking Whitney's theorem by employing Lemma~\ref{injective-dimensional-reduction}.

\item[(iii)]
We provide examples that injective universality does not hold when $L^{2}(D)^{d_{in}}$ and $L^{2}(D)^{d_{out}}$ are replaced by $\R^{d_{in}}$ and $\R^{d_{out}}$:
Consider the case where $d_{in}=d_{out}=1$ and $G^+:\R\to \R$ is defined as $G^+(x)=\sin(x)$.
We can not approximate $G^+:\R\to \R$ by an injective function $G:\R\to \R$ in the set $K=[0,2\pi]$ in the $L^\infty$-norm. 
According to the topological degree theory (see 
\citet[Theorem 1.2.6(iii)]{cho2006topological}), any continuous function $G:\R\to \R$ which satisfies $\|G-G^+\|_{C([0,2\pi])}<\e$ satisfies the equation on both intervals $I_1=[0,\pi]$, $I_2=[\pi,2\pi]$
deg$(G,I_j,s)=$deg$(G^+,I_j,s)=1$ for all $s\in [-1+\e,1-\e]$, $j=1,2$. This implies
that $G:I_j\to \R$ obtains the value $s\in [-1+\e,1-\e]$ at least once.
Hence, $G$ obtains the values $s\in [-1+\e,1-\e]$ at least two times on the interval $[0,2\pi]$ and is it thus not injective.
It is worth noting that the degree theory exhibits significant differences between the infinite-dimensional and finite-dimensional cases \citep{cho2006topological}). 

\end{itemize}
\end{remark}

\section{Details of Section~\ref{Approximation and injectivity via finite-rank operators}}\label{Appendix3}
\subsection{Finite-rank approximation}
We consider linear integral operators $K_{\ell}$ with $L^2$ kernels $k_{\ell}(x,y)$. 
Let $\{\varphi_{k}\}_{k \in \mathbb{N}}$ be an orthonormal basis in $L^{2}(D)$. 
Since $\{\varphi_{k}(y)\varphi_{p}(x)\}_{k,p \in \mathbb{N}}$ is an orthonormal basis of $L^{2}(D \times D)$, integral kernels $k_{\ell} \in L^2(D \times D; \mathbb{R}^{d_{\ell+1} \times d_{\ell} })$ in integral operators $K_{\ell} \in \mathcal{L}(L^{2}(D)^{d_{\ell}}, L^{2}(D)^{d_{\ell+1}})$ has the expansion
\[
k_{\ell}(x,y) = \sum_{k,p \in \mathbb{N}}C^{(\ell)}_{k,p}\varphi_{k}(y)\varphi_{p}(x),
\]
then integral operators $K_{\ell} \in \mathcal{L}(L^{2}(D)^{d_{\ell}}, L^{2}(D)^{d_{\ell+1}})$ take the form
\[
K_{\ell}u(x) = \sum_{k,p \in \mathbb{N}}C^{(\ell)}_{k,p}(u, \varphi_{k}) \varphi_{p}(x), \ u \in L^{2}(D)^{d_{\ell}},
\]
where $C^{(\ell)}_{k,p} \in \mathbb{R}^{d_{\ell+1} \times d_{\ell}}$ whose $(i,j)$-th component $c^{(\ell)}_{k,p,ij}$ is given by
\[
c^{(\ell)}_{k,p,ij}=(k_{\ell, ij}, \varphi_k \varphi_p)_{L^{2}(D\times D)}.
\]
Here, we write $(u, \varphi_{k}) \in \mathbb{R}^{d_{\ell}}$,
\[
(u, \varphi_{k}) =\left( (u_1, \varphi_{k})_{L^{2}(D)}, ..., (u_{d_{\ell}}, \varphi_{k})_{L^{2}(D)} \right).
\]
\medskip

We define $K_{\ell, N} \in \mathcal{L}(L^{2}(D)^{d_{\ell}}, L^{2}(D)^{d_{\ell+1}})$ as the truncated expansion of $K_{\ell}$ by $N$ finite sum, that is,
\[
K_{\ell, N}u(x) := \sum_{k,p \leq N}C^{(\ell)}_{k,p}(u, \varphi_{k}) \varphi_{p}(x).
\]
Then $K_{\ell, N} \in \mathcal{L}(L^{2}(D)^{d_{\ell}}, L^{2}(D)^{d_{\ell+1}})$ is a finite-rank operator with rank $N$.
Furthermore, we have 
\begin{equation}
\left\|K_{\ell} - K_{\ell, N} \right\|_{\mathrm{op}} 
\leq 
\left\|K_{\ell} - K_{\ell, N} \right\|_{\mathrm{HS}} 
=
\left(\sum_{k,p \geq N} \sum_{i,j} |c^{(\ell)}_{k,p,ij}|^{2}\right)^{1/2},
\end{equation}
which implies that as $N \to \infty$,
\[
\left\|K_{\ell} - K_{\ell, N} \right\|_{\mathrm{op}} \to 0.
\]

\subsection{Layerwise injectivity}

We first revisit layerwise injectivity and bijectivity in the case of the finite-rank approximation.
Let $K_{N} :L^{2}(D)^n \to L^{2}(D)^{m}$ be a finite-rank operator defined by
\[
K_{N}u(x) := \sum_{k,p \leq N}C_{k,p}(u, \varphi_{k}) \varphi_{p}(x), \ u \in L^{2}(D)^{n},
\]
where $C_{k,p} \in \mathbb{R}^{m \times n}$ and $(u, \varphi_{p}) \in \mathbb{R}^{n}$ is given by
\[
(u, \varphi_{p}) =\left( (u_1, \varphi_{p})_{L^{2}(D)}, ..., (u_n, \varphi_{p})_{L^{2}(D)} \right).
\]
Let $b_{N} \in L^{2}(D)^n$ be defined by
\[
b_{N}(x) :=\sum_{p \leq N} b_{p} \varphi_{p}(x),
\]
in which $b_{p} \in \mathbb{R}^{m}$.
As analogues of Propositions~\ref{injectivity-ReLU:general} and \ref{injective-activation-chara}, we obtain the following characterization.
\begin{proposition}\label{injectivity-ReLU:finite-rank}
(i) 
The operator 
\[
\relu \circ(K_{N}+b_{N}): (\mathrm{span}\{\varphi_{k}\}_{k \leq N})^{n} \to L^{2}(D)^{m},
\]
is injective if and only if for every $v \in  (\mathrm{span}\{\varphi_{k}\}_{k \leq N})^{n}$,
\[
\left\{
u \in L^{2}(D)^{n} \ \bigl| \ 
\vec{u}_{N} \in \mathrm{Ker}(C_{S,N})
\right\} \cap X(v, K_{N}+b_{N}) \cap (\mathrm{span}\{\varphi_{k}\}_{k \leq N})^{n} =\{0\}.
\]
where $S(v, K_{N}+b_{N}) \subset [m]$ and $X(v, K_{N}+b_{N})$ are defined in Definition~\ref{definition:DSS:general}, and 
\begin{equation}
\vec{u}_{N}:= \left((u, \varphi_{p}) \right)_{p \leq N} \in \mathbb{R}^{Nn}, \ \
C_{S,N}:=\left(C_{k,q}\bigl|_{S(v, K_{N}+b_{N})} \right)_{k,q \in [N]} \in \mathbb{R}^{N|S(v,K_{N}+b_{N})| \times Nn}.
\label{notation-chara-C_S_N-1}
\end{equation}
(ii) Let $\sigma$ be injective. 
Then the operator
\[
\sigma \circ (K_N+b_{N}) : (\mathrm{span}\{\varphi_{k}\}_{k \leq N})^{n}\to L^{2}(D)^{m},
\]
is injective if and only if $C_{N}$ is injective, where
\begin{equation}
C_{N}:=\left(C_{k,q} \right)_{k,q \in [N]} \in \mathbb{R}^{Nm \times Nn}.
\label{notation-chara-C_S_N-2}
\end{equation}
\end{proposition}
\begin{proof}
The above statements follow from Propositions~\ref{injectivity-ReLU:general} and \ref{injective-activation-chara} by observing that $u \in \mathrm{Ker}\left(K_{N}\right)$ is equivalent to (cf.~(\ref{notation-chara-C_S_N-1}) and (\ref{notation-chara-C_S_N-2}))
\[
\sum_{k,p \leq N}C_{k,p} (u, \varphi_{k}) \varphi_{p}=0,
\iff 
C_{N} \vec{u}_{N} = 0.
\]
\end{proof}

\subsection{Global injectivity}

We revisit global injectivity in the case of finite-rank approximation.
As an analogue of Lemma~\ref{injective-dimensional-reduction}, we have the following
\begin{lemma}\label{injective-dimensional-reduction-finite-rank}
Let $N, N^{\prime} \in \mathbb{N}$ and $n,m,\ell \in \mathbb{N}$ with $N^{\prime}m > N^{\prime}\ell \geq 2N n +1 $, and let $T:L^{2}(D)^{n} \to L^{2}(D)^{m}$ be a finite-rank operator with $N^{\prime}$ rank, that is,
\begin{equation}
\mathrm{Ran}(T) \subset (\mathrm{span}\{\varphi_{k}\}_{k \leq N^{\prime}})^{m},
\label{ran-T-finite-rank}
\end{equation}
and Lipschitz continuous, and
\[
T : \left(\mathrm{span}\{\varphi_{k}\}_{k \leq N} \right)^{n} \to L^{2}(D)^{m},
\]
is injective. 
Then, there exists a finite-rank operator $B \in \mathcal{L}(L^{2}(D)^{m}, L^{2}(D)^{\ell})$ with rank $N^{\prime}$ such that
\[
B \circ T : (\mathrm{span}\{\varphi_{k}\}_{k \leq N})^{n} \to (\mathrm{span}\{\varphi_{k}\}_{k \leq N^{\prime}})^{\ell},
\]
is injective.
\end{lemma}
\begin{proof}
From (\ref{ran-T-finite-rank}), $T : L^2(D)^{n} \to L^2(D)^{m}$ has the form of
\[
T(a) = \sum_{k \leq N^{\prime}} (T(a), \varphi_{k})\varphi_{k},
\]
where $(T(a), \varphi_{k}) \in \mathbb{R}^{m}$.
We define $\mathbf{T} : \mathbb{R}^{Nn} \to \mathbb{R}^{N^{\prime}m}$ by
\[
\mathbf{T}(\mathbf{a}):=\left((T(\mathbf{a}), \varphi_{k}) \right)_{k \in [N^{\prime}]} \in \mathbb{R}^{N^{\prime}m}, \ \mathbf{a} \in \mathbb{R}^{Nn},
\]
where $T(\mathbf{a}) \in L^2(D)^{m} $ is defined by
\[
T(\mathbf{a}):=T\left( \sum_{k \leq N}a_{k} \varphi_{k} \right)
\in L^2(D)^{m},
\]
in which $a_{k} \in \mathbb{R}^{n}$, $\mathbf{a}=(a_1,...,a_{N}) \in \mathbb{R}^{Nn}$.

Since $T : L^2(D)^{n} \to L^2(D)^{m}$ is  Lipschitz continuous, $\mathbf{T} : \mathbb{R}^{Nn} \to \mathbb{R}^{N^{\prime}m}$  is also  Lipschitz continuous.
As $N^{\prime}m > N^{\prime}\ell \geq 2N n +1 $, we can apply Lemma 29 from \citet{puthawala2022globally} with $D=N^{\prime}m$, $m=N^{\prime}\ell$, $n=Nn$.
According to this lemma, there exists a $N^{\prime}\ell$-dimensional linear subspace $\mathbf{V}^{\perp}$ in $\mathbb{R}^{N^{\prime}m}$ such that
\[
\left\|P_{\mathbf{V}^{\perp}}-P_{\mathbf{V}^{\perp}_{0}} \right\|_{\mathrm{op}} < 1,
\]
and
\[
P_{\mathbf{V}^{\perp}} \circ \mathbf{T} : \mathbb{R}^{Nn} \to \mathbb{R}^{N^{\prime}m},
\]
is injective, where $\mathbf{V}_{0}^{\perp}=\{0\}^{N^{\prime}(m-\ell)} \times \mathbb{R}^{N^{\prime}\ell}$.
Furthermore, in the proof of Theorem 15 of \citet{puthawala2022globally}, denoting
\[
\mathbf{B}:=\pi_{N^{\prime}\ell} \circ \mathbf{Q} \circ P_{\mathbf{V}^{\perp}} \in \mathbb{R}^{N^{\prime}\ell \times N^{\prime}m},
\]
we are able to show that 
\[
\mathbf{B} \circ \mathbf{T} : \mathbb{R}^{Nn} \to \mathbb{R}^{N^{\prime}\ell}
\]
is injective.
Here, $\pi_{N^{\prime}\ell} : \mathbb{R}^{N^{\prime}m} \to \mathbb{R}^{N^{\prime}\ell}$
\[
\pi_{N^{\prime}\ell}(a, b) := b, \quad (a,b) \in \mathbb{R}^{N^{\prime}(m-\ell)} \times \mathbb{R}^{N^{\prime}\ell},
\]
where $\mathbf{Q} : \mathbb{R}^{N^{\prime}m} \to \mathbb{R}^{N^{\prime}m}$ is defined by
\[
\mathbf{Q} := \left( P_{\mathbf{V}_{0}^{\perp}} P_{\mathbf{V}^{\perp}} + (I-P_{\mathbf{V}_{0}^{\perp}}) (I-P_{\mathbf{V}^{\perp}}) \right)
\left( I-(P_{\mathbf{V}_{0}^{\perp}}-P_{\mathbf{V}^{\perp}})^2 \right)^{-1/2}.
\]
We define $B : L^2(D)^{m} \to L^2(D)^{\ell}$ by
\[
Bu= \sum_{k,p \leq N^{\prime}} \mathbf{B}_{k,p} (u, \varphi_{k}) \varphi_{p},
\]
where $\mathbf{B}_{k,p} \in \mathbb{R}^{\ell \times m}$, $\mathbf{B}=(\mathbf{B}_{k,p})_{k,p \in [N^{\prime}]}$. 
Then $B: L^2(D)^{m} \to L^2(D)^{\ell}$ is a linear finite-rank operator with $N^{\prime}$ rank, and 
\[
B \circ T : L^2(D)^{n} \to L^2(D)^{\ell}
\]
is injective because, by the construction, it is equivalent to
\[
\mathbf{B} \circ \mathbf{T} : \mathbb{R}^{Nn} \to \mathbb{R}^{N^{\prime}\ell},
\]
is injective.
\end{proof}

\medskip

\subsection{Proof of Theorem~\ref{Universal-injectivity-NO-finite-rank}}
\label{proof-Universal-injectivity-NO-finite-rank}

\begin{proof}

Let $R>0$ such that
\[
K \subsetneq B_{R}(0),
\]
where $B_{R}(0):=\{ u \in L^{2}(D)^{d_{in}} \ | \ \left\| u \right\|_{L^{2}(D)^{d_{in}}} \leq R \}$. 
As ReLU and Leaky ReLU function belongs to $\mathrm{A}_{0}^{L} \cap \mathrm{BA}$, 
by Theorem 11 of \citet{kovachki2021neural}, there exists $L \in \mathbb{N}$ and $\widetilde{G} \in \mathrm{NO}_{L}(\sigma ; D, d_{in}, d_{out})$ such that
\begin{equation}
\sup_{a \in K} \left\| G^{+}(a) - \widetilde{G}(a) \right\|_{L^{2}(D)^{d_{out}}} \leq \frac{\epsilon}{3}.
\label{from-main-Thm-finite}
\end{equation}
and
\[
\left\| \widetilde{G}(a) \right\|_{L^{2}(D)^{d_{out}}} \leq 4 M, \quad\hbox{for } a \in L^2(D)^{d_{in}},\quad
\|a\|_{L^2(D)^{d_{in}}}\le R.
\]
We write operator $\widetilde{G}$ by
\[
\widetilde{G} = \widetilde{K}_{L+1} \circ (\widetilde{K}_{L}+\widetilde{b}_{L}) \circ \sigma \cdots  \circ (\widetilde{K}_{2}+\widetilde{b}_{2}) \circ \sigma \circ (\widetilde{K}_{1}+\widetilde{b}_{1}) \circ (\widetilde{K}_{0}+\widetilde{b}_{0}),
\]
where
\[
\begin{split}
&\widetilde{K}_{\ell} \in \mathcal{L}(L^{2}(D)^{d_{\ell}}, L^{2}(D)^{d_{\ell+1}}), \
\widetilde{K}_{\ell}: f \mapsto \int_{D}\widetilde{k}_{\ell}(\cdot, y)f(y)dy, 
\\
&
\widetilde{k}_{\ell} \in L^{2}(D \times D; \mathbb{R}^{d_{\ell+1} \times d_{\ell}}), \ 
\widetilde{b}_{\ell} \in L^{2}(D; \mathbb{R}^{d_{\ell+1} }),
\\
&
d_{\ell} \in \mathbb{N}, \ 
d_{0}=d_{in}, \ d_{L+2}=d_{out}, 
\
\ell=0,...,L+2.
\end{split}
\]

We set $\widetilde{G}_{N^{\prime}} \in \mathrm{NO}_{L,N^{\prime}}(\sigma ; D, d_{in}, d_{out})$ such that
\[
\widetilde{G}_{N^{\prime}} = \widetilde{K}_{L+1,N^{\prime}} \circ (\widetilde{K}_{L,N^{\prime}}+\widetilde{b}_{L,N^{\prime}}) \circ \sigma \cdots  \circ (\widetilde{K}_{2,N^{\prime}}+\widetilde{b}_{2,N^{\prime}}) \circ \sigma \circ (\widetilde{K}_{1,N^{\prime}}+\widetilde{b}_{1,N^{\prime}})\circ (\widetilde{K}_{0,N^{\prime}}+\widetilde{b}_{0,N^{\prime}}),
\]
where $\widetilde{K}_{\ell,N^{\prime}}: L^{2}(D)^{d_{\ell}} \to L^{2}(D)^{d_{\ell+1}}$ is defined by
\[
\widetilde{K}_{\ell,N^{\prime}}u(x) = \sum_{k,p \leq N^{\prime}}C^{(\ell)}_{k,p}(u, \varphi_{k}) \varphi_{p}(x), 
\]
where $C^{(\ell)}_{k,p} \in \mathbb{R}^{d_{\ell+1} \times d_{\ell}}$ whose $(i,j)$-th component $c^{(\ell)}_{k,p,ij}$ is given by
\[
c^{(\ell)}_{k,p,ij}=(\widetilde{k}_{\ell, ij}, \varphi_k \varphi_p)_{L^{2}(D\times D)}.
\]
Since 
\[
\left\|\widetilde{K}_{\ell} - \widetilde{K}_{\ell, N^{\prime}} \right\|_{\mathrm{op}}^{2} 
\leq 
\left\|\widetilde{K}_{\ell} - \widetilde{K}_{\ell, N^{\prime}} \right\|_{\mathrm{HS}}^{2} 
=
\sum_{k,p \geq N^{\prime}+1} \sum_{i,j} |c^{(\ell)}_{k,p,ij}|^{2}
\to 0 \ \text{as} \ N^{\prime} \to \infty,
\]
there is a large $N^{\prime} \in \mathbb{N}$ such that
\begin{equation}
\sup_{a \in K} \left\|\widetilde{G}(a) - \widetilde{G}_{N^{\prime}}(a) \right\|_{L^{2}(D)^{d_{out}}} \leq \frac{\epsilon}{3}.
\label{finite-rank-appro-finite}
\end{equation}
Then, we have 
\[
\begin{split}
\sup_{a \in K} \left\| \widetilde{G}_{N^{\prime}}(a) \right\|_{L^{2}(D)^{d_{out}}} 
&\leq
\sup_{a \in K} \left\| \widetilde{G}_{N^{\prime}}(a)-\widetilde{G}(a) \right\|_{L^{2}(D)^{d_{out}}}
+\sup_{a \in K} \left\|\widetilde{G}(a) \right\|_{L^{2}(D)^{d_{out}}}
\\
&
\leq 1+4M.
\end{split}
\]

We define the operator $H_{N^{\prime}}:L^{2}(D)^{d_{in}} \to L^{2}(D)^{d_{in}+d_{out}}$ by
\[
H_{N^{\prime}}(a)
=
\left(
\begin{array}{c}
H_{N^{\prime}}(a)_{1}\\
H_{N^{\prime}}(a)_{2} \\
\end{array}
\right)
:=
\left(
\begin{array}{c}
K_{inj, N} \circ \cdots \circ K_{inj, N}(a)\\
\widetilde{G}_{N^{\prime}}(a) \\
\end{array}
\right),
\]
where $K_{inj, N} : L^2(D)^{d_{in}} \to L^2(D)^{d_{in}}$ is defined by 
\[
K_{inj, N}u=\sum_{k\leq N}(u, \varphi_{k})\varphi_{k}.
\]
As $K_{inj, N}\left(\mathrm{span}\{\varphi_{k}\}_{k \leq N} \right)^{d_{in}} \to L^2(D)^{d_{in}}$ is injective, 
\[
H_{N^{\prime}} : \left(\mathrm{span}\{\varphi_{k}\}_{k \leq N} \right)^{d_{in}} \to \left(\mathrm{span}\{\varphi_{k}\}_{k \leq N} \right)^{d_{in}} \times 
\left(\mathrm{span}\{\varphi_{k}\}_{k \leq N^{\prime}} \right)^{d_{out}},
\]
is injective.
Furthermore, by the same argument (ii) (construction of $H$) in the proof of Theorem~\ref{Universal-injectivity-NO}, 
\[
H_{N^{\prime}} \in NO_{L,N^{\prime}}(\sigma ; D, d_{in}, d_{out}),
\]
because both of two-layer ReLU and Leaky ReLU neural networks can represent the identity map. Note that above $K_{inj, N}$ is an orthogonal projection, so that
$K_{inj, N} \circ \cdots \circ K_{inj, N}=K_{inj, N}$. However,
we write above $H_{N^{\prime}}(a)_{1}$ as $K_{inj, N} \circ \cdots \circ K_{inj, N}(a)$ so that it can be considered as combination of $(L+2)$ layers of neural networks.

We estimate that for $a \in L^2(D)^{d_{in}}$, $\|a\|_{ L^2(D)^{d_{in}}}\le R$,
\[
\left\| H_{N^{\prime}}(a) \right\|_{L^{2}(D)^{d_{in}+d_{out}}}
\leq
1 + 4M + \left\|K_{inj} \right\|_{\mathrm{op}}^{L+2}R =: C_{H}.
\]
Here, we repeat an argument similar to the one in the proof of Lemma~\ref{injective-dimensional-reduction-finite-rank}:
$H_{N^{\prime}} : L^2(D)^{d_{in}} \to L^2(D)^{d_{in}+d_{out}}$ has the form of
\[
H_{N^{\prime}}(a) = \left( \sum_{k \leq N} (H_{N^{\prime}}(a)_{1}, \varphi_{k})\varphi_{k},
\
\sum_{k \leq N^{\prime}} (H_{N^{\prime}}(a)_{2}, \varphi_{k})\varphi_{k}
\right). 
\]
where $(H_{N^{\prime}}(a)_1, \varphi_{k}) \in \mathbb{R}^{d_{in}}$, $(H_{N^{\prime}}(a)_2, \varphi_{k}) \in \mathbb{R}^{d_{out}}$.
We define $\mathbf{H}_{N^{\prime}} : \mathbb{R}^{Nd_{in}} \to \mathbb{R}^{Nd_{in}+N^{\prime}d_{out}}$ by
\[
\mathbf{H}_{N^{\prime}}(\mathbf{a}):= \left[ \left((H_{N^{\prime}}(\mathbf{a})_{1}, \varphi_{k}) \right)_{k \in [N]}, \  \left((H_{N^{\prime}}(\mathbf{a})_{2}, \varphi_{k}) \right)_{k \in [N^{\prime}]} \right] \in \mathbb{R}^{Nd_{in}+N^{\prime}d_{out}} , \ \mathbf{a} \in \mathbb{R}^{Nd_{in}},
\]
where $H_{N^{\prime}}(\mathbf{a}) = (H_{N^{\prime}}(\mathbf{a})_1, H_{N^{\prime}}(\mathbf{a})_2) \in L^2(D)^{d_{in}+d_{out}} $ is defined by
\[
H_{N^{\prime}}(\mathbf{a})_{1}:=H_{N^{\prime}}\left( \sum_{k \leq N}a_{k} \varphi_{k} \right)_{1}
\in L^2(D)^{d_{in}},
\]
\[
H_{N^{\prime}}(\mathbf{a})_{2}:=H_{N^{\prime}}\left( \sum_{k \leq N^{\prime}}a_{k} \varphi_{k} \right)_{2}
\in L^2(D)^{d_{out}},
\]
where $a_{k} \in \mathbb{R}^{d_{in}}$, $\mathbf{a}=(a_1,...,a_{N}) \in \mathbb{R}^{Nd_{in}}$.
Since $H_{N^{\prime}} : L^2(D)^{d_{in}} \to L^2(D)^{d_{in}+d_{out}}$ is Lipschitz continuous, $\mathbf{H}_{N^{\prime}} : \mathbb{R}^{Nd_{in}} \to \mathbb{R}^{N^{\prime}d_{out}}$ is also Lipschitz continuous.
As 
\[
Nd_{in}+N^{\prime}d_{out} > N^{\prime}d_{out} \geq 2Nd_{in} +1,
\]
we can apply Lemma 29 of \citet{puthawala2022globally} with $D=Nd_{in}+N^{\prime}d_{out}$, $m=N^{\prime}d_{out}$, $n=Nd_{in}$.
According to this lemma, there exists a $N^{\prime}d_{out}$-dimensional linear subspace $\mathbf{V}^{\perp}$ in $\mathbb{R}^{Nd_{in}+N^{\prime}d_{out}}$ such that
\[
\left\|P_{\mathbf{V}^{\perp}}-P_{\mathbf{V}^{\perp}_{0}} \right\|_{op} < \min\left(\frac{\epsilon}{15C_{H_N}}, \ 1 \right)=:\epsilon_{0}
\]
and
\[
P_{\mathbf{V}^{\perp}} \circ \mathbf{H}_{N^{\prime}} : \mathbb{R}^{Nd_{in}} \to \mathbb{R}^{Nd_{in}+N^{\prime}d_{out}},
\]
is injective, where $\mathbf{V}^{\perp}_{0}=\{0\}^{Nd_{in}} \times \mathbb{R}^{N^{\prime}d_{out}}$.
Furthermore, in the proof of Theorem 15 of \citet{puthawala2022globally}, denoting by
\[
\mathbf{B}:=\pi_{N^{\prime}d_{out}} \circ \mathbf{Q} \circ P_{\mathbf{V}^{\perp}},
\]
we can show that 
\[
\mathbf{B} \circ \mathbf{H}_{N^{\prime}} : \mathbb{R}^{Nd_{in}} \to \mathbb{R}^{N^{\prime}d_{out}},
\]
is injective,
where $\pi_{N^{\prime}d_{out}} : \mathbb{R}^{Nd_{in}+N^{\prime}d_{out}} \to \mathbb{R}^{N^{\prime}d_{out}}$
\[
\pi_{N^{\prime}d_{out}}(a, b) := b, \quad (a,b) \in \mathbb{R}^{Nd_{in}} \times \mathbb{R}^{N^{\prime}d_{out}},
\]
and $\mathbf{Q} : \mathbb{R}^{Nd_{in}+N^{\prime}d_{out}} \to \mathbb{R}^{Nd_{in}+N^{\prime}d_{out}}$ is defined by
\[
\mathbf{Q} := \left( P_{\mathbf{V}_{0}^{\perp}} P_{\mathbf{V}^{\perp}} + (I-P_{\mathbf{V}_{0}^{\perp}}) (I-P_{\mathbf{V}^{\perp}}) \right)
\left( I-(P_{\mathbf{V}_{0}^{\perp}}-P_{\mathbf{V}^{\perp}})^2 \right)^{-1/2}.
\]
By the same argument in proof of Theorem 15 in \citet{puthawala2022globally}, we can show that
\[
\left\|I-\mathbf{Q} \right\|_{\mathrm{op}} \leq 4 \epsilon_{0}.
\]
We define $B : L^2(D)^{d_{in}+d_{out}} \to L^2(D)^{d_{out}}$
\[
Bu= \sum_{k,p \leq N^{\prime}} \mathbf{B}_{k,p} (u, \varphi_{k}) \varphi_{p},
\]
$\mathbf{B}_{k,p} \in \mathbb{R}^{d_{out} \times (d_{in}+d_{out})}$, $\mathbf{B}=(\mathbf{B}_{k,p})_{k,p \in [N^{\prime}]}$, then $B: L^2(D)^{d_{in}+d_{out}} \to L^2(D)^{d_{out}}$ is a linear finite-rank operator with $N^{\prime}$ rank.
Then, 
\[
G_{N^{\prime}}:=B \circ H_{N^{\prime}} : L^2(D)^{d_{in}} \to L^2(D)^{d_{out}},
\]
is injective because by the construction, it is equivalent to
\[
\mathbf{B} \circ \mathbf{H}_{N^{\prime}} : \mathbb{R}^{Nd_{in}} \to \mathbb{R}^{N^{\prime}d_{out}},
\]
is injective.
Furthermore,  we have
\[
G_{{N^{\prime}}} \in NO_{L,{N^{\prime}}}(\sigma ; D, d_{in}, d_{out}).
\]
Indeed, $H_{N^{\prime}} \in NO_{L,N^{\prime}}(\sigma ; D, d_{in}, d_{out})$, $B$ is the linear finite-rank operator with $N^{\prime}$ rank, and multiplication of two linear finite-rank operators with $N^{\prime}$ rank is also a linear finite-rank operator with $N^{\prime}$ rank.

\medskip

Finally, we estimate for $a \in K$,
\begin{equation}
\begin{split}
&\left\|G^{+}(a) - G_{N^{\prime}}(a) \right\|_{L^{2}(D)^{d_{out}}}
\\
&
= \underbrace{\left\|G^{+}(a) - \widetilde{G}(a) \right\|_{L^{2}(D)^{d_{out}}}}
_{(\ref{from-main-Thm-finite})\leq \frac{\epsilon}{3}}
+
\underbrace{\left\|\widetilde{G}(a) - \widetilde{G}_{N^{\prime}}(a) \right\|_{L^{2}(D)^{d_{out}}}}
_{(\ref{finite-rank-appro-finite})\leq \frac{\epsilon}{3}}
+
\left\|\widetilde{G}_{N^{\prime}}(a) - G_{N^{\prime}}(a) \right\|_{L^{2}(D)^{d_{out}}}.
\end{split}
\label{estimate-G^+-G-1-finite}
\end{equation}
Using notation $(a, \varphi_{k}) \in \mathbb{R}^{d_{in}}$, and
$
\mathbf{a}=\left((a, \varphi_{k})\right)_{k \in [N]} \in \mathbb{R}^{Nd_{in}},
$
we further estimate for $a \in K$,
\begin{equation}
\begin{split}
& \left\|\widetilde G_{N^{\prime}}(a) -  G_{N^{\prime}}(a) \right\|_{L^{2}(Q)^{d_{out}}}
=\left\|\pi_{d_{out}} H_{N^{\prime}}(a) -  B \circ H_{N^{\prime}}(a) \right\|_{L^{2}(Q)^{d_{out}}}
\\
&
=\left\|\pi_{N^{\prime}d_{out}} \mathbf{H}_{N^{\prime}}(\mathbf{a}) -  \mathbf{B} \circ \mathbf{H}_{N^{\prime}}(\mathbf{a}) \right\|_{2}
\\
&
= \left\| \pi_{N^{\prime}d_{out}} \circ \mathbf{H}_{N^{\prime}}(\mathbf{a}) - \pi_{N^{\prime}d_{out}} \circ \mathbf{Q} \circ P_{\mathbf{V}^{\perp}} \circ \mathbf{H}_{N^{\prime}}(\mathbf{a}) \right\|_{2}
\\
&
\leq 
\left\| \pi_{N^{\prime}d_{out}} \circ (P_{\mathbf{V}^{\perp}_{0}}-P_{\mathbf{V}^{\perp}}+P_{\mathbf{V}^{\perp}}) \circ \mathbf{H}_{N^{\prime}}(\mathbf{a}) - \pi_{N^{\prime}d_{out}} \circ \mathbf{Q} \circ P_{\mathbf{V}^{\perp}} \circ \mathbf{H}_{N^{\prime}}(\mathbf{a}) \right\|_{2}
\\
&
\leq 
\left\| \pi_{N^{\prime}d_{out}} \circ (P_{\mathbf{V}_{0}^{\perp}}-P_{\mathbf{V}^{\perp}}) \circ \mathbf{H}_{N^{\prime}}(\mathbf{a}) \right\|_{2}
+\left\|\pi_{N^{\prime}d_{out}} \circ (I - \mathbf{Q}) \circ P_{\mathbf{V}^{\perp}} \circ \mathbf{H}_{N^{\prime}}(\mathbf{a}) \right\|_{2}
\\
&
\leq 5 \epsilon_{0} 
\underbrace{\left\|\mathbf{H}_{N^{\prime}}(\mathbf{a})  \right\|_{2}}
_{=\left\|H_{N^{\prime}}(a) \right\|_{L^2(D)^{d_{out}}} < C_{H} }
\leq \frac{\epsilon}{3},
\end{split}
\label{estimate-G^+-G-2-finite}
\end{equation}
where $\left\|\cdot \right\|_{2}$ is the Euclidean norm. 
Combining (\ref{estimate-G^+-G-1-finite}) and (\ref{estimate-G^+-G-2-finite}), we conclude that
\[
\sup_{a \in K} \left\| 
G^{+}(a) - G_{N^{\prime}}(a) \right\|_{L^{2}(D)^{d_{out}}} \leq \frac{\epsilon}{3}+\frac{\epsilon}{3} + \frac{\epsilon}{3}= \epsilon.
\]

\end{proof}

\section{Details of Section~\ref{Surjectivity and bijectivity}}
\label{Appendix4}

\subsection{Proof of Proposition~\ref{lem:surjectivity-of-injective-nonlinear-operators}}
\label{proof-lem:surjectivity-of-injective-nonlinear-operators}
\begin{proof}
Since $W$ is bijective, and $\sigma$ is surjective, it is enough to show that $u \mapsto Wu + K(u)$ is surjective.
We observe that for $z \in L^2(D)^n$,
\[
Wu + K(u) = z,
\]
is equivalent to
\[
H_{z}(u):=-W^{-1}K(u) + W^{-1} z = u.
\]
We will show that $H_{z}:L^2(D)^{n} \to L^2(D)^{n}$ has a fixed point for each $z \in L^2(D)^n$.
By the Leray-Schauder theorem, see \citet[Theorem 11.3]{Gilbarg}, $H:L^2(D)\to L^2(D)$ has a fixed point if the union  $\bigcup_{0< \lambda \le 1} V_\lambda$ is bounded, where the sets
\[
\begin{split}
V_\lambda&:=\{u\in L^2(D):\ u=\lambda H_{z}(u)\}\\
&=\{u\in L^2(D):\  \lambda ^{-1} u= H_{z}(u)\}\\
&=\{u\in L^2(D):\  -\lambda ^{-1} u= W ^{-1}K (u)-W^{-1}z\},
\end{split}
\]
are parametrized by  $0< \lambda \le 1$.

As the map $u \mapsto \alpha u+W ^{-1}K(u)$ is coercive, there is an $r>0$ such that for $\|u\|_{L^2(D)^{n}} > r$,
\[
\frac {\left< \alpha u + W ^{-1}K(u), u \right>_{L^2(D)^n}}{\|u\|_{L^2(D)^n}} \geq \|W^{-1}z\|_{L^2(D)^{n}}.
\]
Thus, we have that for $\|u\|_{L^2(D)^{n}} > r$
\[
\begin{split}
&\frac{\left<W ^{-1}K(u) - W^{-1}z, u \right>_{L^2(D)^n}}{\|u\|^{2}_{L^2(D)^n}} 
\\
&
\geq 
\frac{\left<\alpha u + W ^{-1}K(u), u \right>_{L^2(D)^n}-\left<\alpha u + W ^{-1}z, u \right>_{L^2(D)^n}}{\|u\|^{2}_{L^2(D)^n}}
\\
&
\geq 
\frac{\|W^{-1}z\|_{L^2(D)^{n}}}{\|u\|_{L^2(D)^n}} - \frac{\left< W ^{-1}z, u \right>_{L^2(D)^n}}{\|u\|^{2}_{L^2(D)^n}} - \alpha
\geq - \alpha > -1,
\end{split}
\]
and, hence, for all $\|u\|_{L^2(D)}>r_0$ and $\lambda\in (0,1]$ we have $u\not \in V_\lambda$.
Thus
\[
\bigcup_{\lambda\in (0,1]}V_\lambda\subset B(0,r_0). 
\] 
Again, by the Leray-Schauder theorem (see \citet[Theorem 11.3]{Gilbarg}), $H_{z}$ has a fixed point.
\end{proof}
\medskip

\subsection{Examples for Proposition~\ref{lem:surjectivity-of-injective-nonlinear-operators}}

\begin{example}\label{example-layer-sur-app}
We consider the case where $n=1$ 
and $D \subset \mathbb{R}^{d}$ is a bounded interval. 
We consider the non-linear integral operator,
\[
K(u)(x):=\int_{D} k(x,y,u(x))u(y)dy, \ x \in D,
\]
and $k(x,y,t)$ is bounded, that is, there is $C_{K}>0$ such that
\[
\left|k(x,y,t) \right| \leq C_{K}, \ x,y \in D, \ t \in \mathbb{R}.
\]
If $\left\|W^{-1} \right\|_{\mathrm{op}}$ is small enough such that 
\[
1 > \left\|W^{-1} \right\|_{\mathrm{op}}C_{K}|D|, 
\]
then, for $\alpha \in \left(\left\|W^{-1} \right\|_{\mathrm{op}}C_{K}|D|,1 \right)$, $u \mapsto \alpha u+W ^{-1}K(u)$ is coercive.
Indeed, {we have for $u \in L^2(D)$,
\[
\begin{split}
&\frac{\left<\alpha u+W ^{-1}K(u), u \right>_{L^2(D)}}{\|u\|_{L^2(D)}} 
\\
&
\geq 
\alpha \|u\|_{L^2(D)} - \left\|W^{-1} \right\|_{\mathrm{op}} \|K(u)\|_{L^2(D)}
\geq 
\underbrace{
\left( \alpha - \left\|W^{-1} \right\|_{\mathrm{op}}C_{K}|D| \right)  
}_{> 0}
\|u\|_{L^2(D)}.
\end{split}
\]
For example, we can consider a kernel
\[
k(x,y,t)=\sum_{j=1}^{J} c_{j}(x,y)\sigma_{s}(a_j(x,y)t+b_j(x,y)),
\]
where $\sigma_{s}:\R\to \R$ is the sigmoid function defined by
\[
\sigma_{s}(t)=\frac 1{1+e^{-t}}.
\]
There are functions $a,b,c \in C(\overline D\times \overline D)$ such that 
\ba
\sum_{j=1}^J\|c_j\|_{L^{\infty}(D\times D)}< \left\|W^{-1} \right\|_{\mathrm{op}}^{-1}|D|^{-1}. 
\ea
}

\end{example}

\begin{example}\label{example-layer-sur-appB}
Again, we consider the case where $n=1$ and $D \subset \mathbb{R}^{d}$ is a bounded set.
We assume that $W\in C^1(\overline D)$ satisfies $0< c_1\le W(x)\le c_2$.
For simplicity, we assume that $|D|=1$. 
We consider the non-linear integral operator 
\beq\label{e: K oper}
K(u)(x):=\int_{D} k(x,y,u(x))u(y)dy, \ x \in D,
\eeq
where
\beq\label{e: K oper kernel}
k(x,y,t)=\sum_{j=1}^{J} c_{j}(x,y)\sigma_{wire}(a_j(x,y)t+b_j(x,y)),
\eeq
in which $\sigma_{wire}:\R\to \R$ is the wavelet function defined by
\[
\sigma_{wire}(t)=\hbox{Im}\,(e^{i\omega t}e^{-t^2}),
\]
and $a_j,b_j,c_j \in C(\overline D\times \overline D)$ are such that the $a_j(x,y)$ are nowhere
vanishing functions, that is, $a_j(x,y)\not =0$ for all $x,y\in\overline D\times \overline D$. 

The next lemma holds for any activation function with exponential decay, including the activation function $\sigma_{wire}$ and settles the key condition for Proposition~\ref{lem:surjectivity-of-injective-nonlinear-operators} to hold.

\begin{lemma}
Assume that $|D|=1$ and the activation function $\sigma:\R\to \R$ is continuous.
Assume that there exists $M_1,m_0>0$ such that
\ba
|\sigma(t)|\leq  M_1e^{-m_0|t|},\quad  t \in \mathbb{R}.
\ea
Let $a_j,b_j,c_j \in C(\overline D\times \overline D)$ be such that $a_j(x,y)$ are nowhere
vanishing functions. 
Moreover, let $K:L^2(D)\to L^2(D)$ be a non-linear integral operator given in \eqref{e: K oper}
with a kernel satisfying \eqref{e: K oper kernel},
$\alpha>0$ and $0<c_0\le W(x)\le c_1$
. Then function $F:L^2(D)\to L^2(D)$, $F(u)=\alpha u+W ^{-1}K(u)$ is coercive.
\end{lemma}

\begin{proof}
As $\overline D$
is compact, there is $a_0>0$ such that  for all $j=1,2,\dots,J$ we have
$|a_j(x,y)|\ge a_0$ a.e. and  $|b_j(x,y)|\le b_0$ a.e.
We point out that $|\sigma(t)|\leq M_1$.
Next, let $\e>0$ be such that 
\begin{equation}
(\sum_{j=1}^J\|W^{-1}c_j\|_{L^\infty(D\times D)}) M_1 \e<\frac \alpha 4,
\label{estimate-lemm4-1}
\end{equation}
$\lambda>0$, and $u\in L^2(D)$. We define the sets
\ba
& &D_1(\lambda)=\{x\in D:\ |u(x)|\ge  \e\lambda\},\\
& &D_2(\lambda)=\{x\in D:\ |u(x)|<  \e \lambda\}.
\ea
Then,
 for $x\in D_2(\lambda)$, 

\[
\begin{split}
&\sum_{j=1}^J\|W^{-1}c_j\|_{L^\infty(D\times D)}|\sigma(a_j(x,y)u(x)+b_j(x,y))u(x)|
\\
&
\le \sum_{j=1}^J\|W^{-1}c_j\|_{L^\infty(D\times D)} M_1 \epsilon \lambda 
\underset{(\ref{estimate-lemm4-1})}{\leq}
\frac \alpha 4 \lambda.
\end{split}
\]
After $\e$ is chosen as in the above, we choose $\lambda_0\ge \max(1,b_0/(a_0  \e))$  to be sufficiently large so that for all
$|t|\ge \e\lambda_0$  it holds that
\ba
\paren{\sum_{j=1}^J\|W^{-1}c_j\|_{L^\infty(D\times D)}} M_1\hbox{exp}(-m_0|a_0t-b_0|) t<\frac \alpha 4.
\ea
Here, we observe that, as $\lambda_0\ge b_0/(a_0  \e)$, we have that for all $|t|\ge \e\lambda_0$, $a_0|t|-b_0>0$.
Then, when $\lambda \ge \lambda_0$, we have  for $x\in D_1(\lambda)$, 
\ba
\paren{\sum_{j=1}^J\|W^{-1}c_j\|_{L^\infty(D\times D})}\bigg|\sigma\bigg(a_j(x,y)u(x)+b_j(x,y)\bigg)u(x)\bigg|\le 
\frac \alpha 4.
\ea

When $u \in L^2(D)$
 has the norm
$\|u\|_{L^2(D)}=\lambda\ge \lambda_0\ge 1$, we have
\ba
& &\left|\int_{D} \int_{D} W(x)^{-1}  k(x,y,u(x))u(x)u(y)dxdy\right|\\
&\le&
\int_{D}\bigg( \int_{D_1} (\sum_{j=1}^J\|W^{-1}c_j\|_{L^\infty(D\times D)})M_1\hbox{exp}\bigg(-m_0|a_0|u(x)|-b_0|\bigg)
 |u(x)|dx\bigg)|u(y)|dy
 \\
& & +\int_{D} \bigg(\int_{D_2}  \sum_{j=1}^J\|W^{-1}c_j\|_{L^\infty(D\times D)}|\sigma(a_j(x,y)u(x)+b_j(x,y))| |u(x)|dx\bigg)|u(y)|dy
\\
&\le& 
\frac  \alpha 4\|u\|_{L^2(D)}+\frac \alpha{4}\lambda\|u\|_{L^2(D)}\\
&\le& 
\frac  \alpha 2\|u\|_{L^2(D)}^2.
\ea
Hence,
\ba
\frac{\left<\alpha u+W ^{-1}K(u), u \right>_{L^2(D)}}{\|u\|_{L^2(D)}} 
\geq \frac  \alpha 2\|u\|_{L^2(D)},
\ea
and the function $u\to \alpha u+W ^{-1}K(u)$ is coercive.
\end{proof}

\end{example}

\medskip

\subsection{Proof of Proposition~\ref{injecitivity-non-linear-NOs}}
\label{proof-injecitivity-non-linear-NOs}

\begin{proof}
(Injectivity) Assume that 
\[
\sigma (Wu_1 + K(u_1)+b) = \sigma (Wu_2 + K(u_2)+b).
\]
where $u_1, u_2 \in L^2(D)^n$.
Since $\sigma$ is injective and $W:L^{2}(D)^n \to L^{2}(D)^n$ is bounded linear bijective, we have 
\[
u_1 + W^{-1}K(u_1) = u_2 + W^{-1}K(u_2) =: z. 
\]
Since the mapping $u \mapsto z - W^{-1}K(u)$ is contraction (because $W^{-1}K$ is contraction),
by the Banach fixed-point theorem, the mapping $u \mapsto z - W^{-1}K(u)$ admit a unique fixed-point in $L^2(D)^n$, which implies that $u_1=u_2$.
\par
\noindent
(Surjectivity) Since $\sigma$ is surjective, it is enough to show that $u \mapsto Wu + K(u)+b$ is surjective. 
Let $z \in L^2(D)^n$.
Since the mapping $u \mapsto W^{-1}z - W^{-1}b - W^{-1}K(u)$ is contraction, by Banach fixed-point theorem, there is $u^{\ast} \in L^2(D)^n$ such that
\[
u^{\ast} = W^{-1}z - W^{-1}b - W^{-1}K(u^{\ast}) 
\iff
Wu^{\ast} + K(u^{\ast}) + b = z.
\]
\end{proof}

\subsection{Examples for Proposition~\ref{injecitivity-non-linear-NOs}}

\begin{example}\label{ex: Volterra}
We consider the case of $n=1$, and $D\subset [0,\ell]^d$. 
We consider Volterra operators 
\ba
K(u)(x)=\int_D k(x,y,u(x),u(y))u(y)dy,
\ea 
where 
 $x=(x_1,\dots,x_d)$ and $y=(y_1,\dots,y_d)$.
We recall that $K$ is a Volterra operator 
if  
\beq\label{Volterra property}
k(x,y,t,s)\not =0\implies 
\hbox{ $y_j\le x_j$\quad for all $j=1,2,\dots,d$.}
\eeq
In particular, when $D=(a,b)\subset \R$ is an interval,
the Volterra operators are of the form
\ba
K(u)(x)=\int_a^x k(x,y,u(x),u(y))u(y)dy,
\ea 
and if $x$ is considered as a time variable, the Volterra operators
are causal in the sense that the value of $K(u)(x)$ at the time $x$ depends only on $u(y)$ at the times $y\le x$.

Assume that  $k(x,y,t,s)\in C(\overline D\times \overline D\times \R\times \R)$ is bounded and uniformly Lipschitz smooth in the $t$ and $s$ variables, that is, $k\in C(\overline D\times \overline D;C^{0,1}(\R\times \R))$.

Next, we consider the non-linear operator $F: L^2(D)\to  L^2(D)$,
\beq
F(u)= u+K(u).
\eeq
Assume that $u,w\in L^2(D)$ are such that $u+K(u)=w+K(w)$,
so that $w-u=K(u)-K(w)$. Next, we will show that then $u=w$.
We denote
and $D(z_1)=D\cap ([0,z_1]\times [0,\ell]^{d-1})$ and 
\[
\|k\|_{C(\overline D\times \overline D;C^{0,1}(\R\times \R))}
:= \sup_{x,y \in D} \left\| k(x,y,\cdot,\cdot) \right\|_{C^{0,1}(\mathbb{R}\times\mathbb{R})},
\]
\[
\|k\|_{L^{\infty}(D\times D \times \mathbb{R}\times\mathbb{R})}
:=\sup_{x,y \in D, s,t \in \mathbb{R}} \left| k(x,y,s,t) \right|.
\]
Then for $x\in D(z_1)$ the Volterra property of the kernel implies that 

\ba
&&\hspace{-7mm}|u(x)-w(x)|\le \int_D|k(x,y,u(x),u(y))u(y)-k(x,y,w(x),w(y))w(y)|dy
\\
&&\hspace{-7mm}  \leq \int_{D(z_1)}|k(x,y,u(x),u(y))u(y)-k(x,y,w(x),u(y))u(y)|dy
\\
&&
+\int_{D(z_1)}|k(x,y,w(x),u(y))u(y)-k(x,y,w(x),w(y))u(y)|dy\\
\\
&&
+\int_{D(z_1)}|k(x,y,w(x),w(y))u(y)-k(x,y,w(x),w(y))w(y)|dy\\
&&\hspace{-7mm}\leq
2\|k\|_{C(\overline D\times \overline D;C^{0,1}(\R \times \R))}\|u-w\|_{L^2(D(z_1))} \|u\|_{L^2(D(z_1))}
\\
&& \hspace{2cm}
+\|k\|_{L^{\infty}(D\times D \times \mathbb{R}\times\mathbb{R})}\|u-w\|_{L^2(D(z_1))}\sqrt{|D(z_1)|},
\ea
so that for all $0<z_1<\ell$, 
\ba
&&\hspace{-7mm}
\|u-w\|_{L^2(D(z_1))}^2\\
&&\hspace{-7mm}=\int_0^{z_1}
\bigg(\int_0^\ell \dots \int_0^\ell {\bf 1}_{D(x)} |u(x)-w(x)|^2dx_ddx_{d-1}\dots dx_2\bigg)dx_1
\\
&&\hspace{-7mm}\leq
z_1\ell^{d-1}\bigg(2\|k\|_{C(\overline D\times \overline D;C^{0,1}(\R\times \R))}\|u-w\|_{L^2(D(z_1))} \|u\|_{L^2(D(z_1))}
\\
&& \hspace{2cm}
+\|k\|_{L^{\infty}(D\times D \times \mathbb{R}\times\mathbb{R})}\|u-w\|_{L^2(D(z_1))}\sqrt{|D(z_1)|}\bigg)^2
\\
&&\hspace{-7mm}\leq
z_1\ell^{d-1} \bigg(\|k\|_{C(\overline D\times \overline D;C^{0,1}(\R \times \R))} \|u\|_{L^2(D)}
+\|k\|_{L\infty (D\times D\times\R \times \R)}\sqrt{|D|}\bigg)^2 
\|u-w\|_{L^2(D(z_1))}^2.
\ea
Thus, when $z_1$ is so small that 
\ba
z_1\ell^{d-1} \bigg(\|k\|_{C(\overline D\times \overline D;C^{0,1}(\R^n))} \|u\|_{L^2(D)}
+\|k\|_{L\infty ( D\times D\times\R \times \R)}\sqrt{|D|}\bigg)^2<1,
\ea
we find that $ \|u-w\|_{L^2(D(z_1))}=0$, that is, $u(x)-w(x)=0$ for $x\in D(z_1)$.
Using the same arguments as above, we see for all $k\in \mathbb N$ that that if  $u=w$
in $D(kz_1)$ then $u=w$
in $D((k+1)z_1)$. Using induction, we see that $u=w$ in $D$.
Hence, the  operator $u\mapsto F(u)$ is injective in $L^2(D)$.

\end{example}

\begin{example}\label{ex: Volterra derivative}
We consider derivatives of Volterra operators in the domain
 $D\subset [0,\ell]^d$. 
Let  $K: L^2(D)\to  L^2(D)$ be a non-linear operator
\beq\label{Volterra type 2}
K(u)=\int_D k(x,y,u(y))u(y)dy,
\eeq 
where  $k(x,y,t)$ satisfies \eqref{Volterra property}, is bounded, and  $k\in C(\overline D\times \overline D;C^{0,1}(\R\times \R))$.
Let  $F_1: L^2(D)\to  L^2(D)$  
be
\beq
F_1(u)=u+K(u).
\eeq
Then the Fr\'{e}chet derivative of $K$ at $u_0\in L^2(D)$ to the direction  $w\in L^2(D)$ is
\beq
DF_1|_{u_0}(w)=w(x)+\int_D k_1(x,y,u_0(y))w(y)dy,
\eeq
where 
\beq
 k_1(x,y,u_0(y))=u_0(y)\frac {\p }{\p t}k(x,y,t)\bigg|_{t=u_0(x)}+k(x,y,u_0(y))
\eeq
is a Volterra operator satisfying
\beq\label{Volterra property2}
k_1(x,y,t)\not =0\implies 
\hbox{ $y_j\le x_j$\quad for all $j=1,2,\dots,d$.}
\eeq
As seen in Example \ref{ex: Volterra}, the operator $DF_1|_{u_0}: L^2(D)\to  L^2(D)$ is injective.

\end{example}

\section{Details of Section~\ref{Construction of the inverse of a non-linear integral neural operator}}
\label{Appendix5}

In this appendix, we prove  Theorem \ref{thm: invertibility}.
We recall that in that theorem, we
consider the case when $n=1$, $D\subset \R$ is a bounded interval, and the operator $F_1$ is of the form
\ba
F_1(u)(x)=W(x) u(x)+\int_Dk(x,y,u(y)) u(y)dy,
\ea
where $W\in C^1(\overline D)$ satisfies $0< c_1\le W(x)\le c_2$,
the function $(x,y,s) \mapsto k(x,y,s)$
is in $C^3(\overline D \times \overline D \times \R)$, and that 
in $\overline D \times \overline D \times \R$ its three derivatives and the derivatives of $W$ are all uniformly bounded
by $c_0$, that is,
\beq\label{eq: kernel k}
\|k\|_{C^3(\overline D \times \overline D \times \R)}\leq c_0,\quad
\|W\|_{C^1(\overline D)}\leq c_0.
\eeq
We recall that the identical embedding $H^1(D)\to  L^\infty(D)$ is bounded and compact by Sobolev's embedding theorem.

As we will consider kernels
$k(x,y,u_0(y))$, we will consider the non-linear operator $F_1$ mainly as an operator in 
a Sobolev space  $H^1(D)$. 

The Frechet derivative of $F_1$ at $u_0$  to direction $w$, denoted by
$A_{u_0}w= DF_1|_{u_0}(w)$ is given by
\beq\label{A-operator}
A_{u_0}w=W(x) w(x)+\int_Dk(x,y,u_0(y)) w(y)dy
+\int_D{u_0}(y)\,\frac {\p k}{\p u}(x,y,u_0(y))w(y) dy.
\eeq
The condition \eqref{eq: kernel k} implies that 
\beq\label{F1 operators}
 F_1:H^1(D)\to H^1(D),
\eeq
is a locally Lipsichitz smooth function and that the operator 
\ba
A_{u_0}:H^1(D)\to H^1(D),
\ea
given in \eqref{A-operator}, is defined for all $u_0\in C(\overline D)$ as a bounded linear operator.

When $\mathcal X$ is a Banach space, we let 
$B_{\mathcal X}(0,R)=\{v\in \mathcal X: \ \|v\|_{\mathcal X}< R\}$ and  $\overline B_{\mathcal X}(0,R)=\{v\in \mathcal X:\  \|v\|_{\mathcal X}\leq R\}$
be the open and closed balls in $\mathcal X$, respectively. 

We consider the 
H\"older spaces $C^{n,\alpha}(\overline D)$ and their image in (leaky) ReLU-type functions.
Let $a\ge 0$ and $\sigma_a(s)=\relu(s)-a\relu(-s)$.
We will consider the image of the closed ball of $C^{1,\alpha}(\overline D)$ in the map $\sigma_a$,
that is $\sigma_a(\overline B_{C^{1,\alpha}(\overline D)}(0,R))\coloneqq\{\sigma_a\circ g\in C(\overline D):\
\|g\|_{C^{1,\alpha}(\overline D)}\leq R\}$.

We will below assume that for  all $u_0\in C(\overline D)$
the integral operator satisfies 
\beq\label{A is injective}
A_{u_0}:H^1(D)\to H^1(D)\hbox{  is an injective operator}.
\eeq
This condition is valid when $K(u)$ is a  Volterra operator, see Examples \ref{ex: Volterra} and
\ref{ex: Volterra derivative}.
As the integral operators $A_{u_0}$ are Fredholm operators having
index zero. This implies that the operators \eqref{A is injective} are bijective.

The inverse operator $A_{u_0}^{-1}:H^1(D)\to H^1(D)$ can be written as
\beq
\label{inverse-A_u_0}
A_{u_0}^{-1}v(x)=\tilde W(x)v(x)-\int_D\tilde k_{u_0}(x,y)v(y)dy,
\eeq
where $\tilde k_{u_0},\p_x \tilde k_{u_0}\in C(\overline D\times \overline D)$ and $\tilde W \in C^1(\overline D)$.

We will consider the inverse function of the map $F_1$ in a set $\mathcal Y\subset \sigma_a(\overline B_{C^{1,\alpha}(\overline D)}(0,R))$ that is a compact subset of the Sobolev space $H^1(D)$.
To this end, we will cover the set $\mathcal Y$ with small balls $B_{H^1(D)}(g_j,\e_0)$,
$j=1,2,\dots,J$ of $H^1(D)$, centered at 
 $g_j=F_1(v_j)$, where $v_j\in H^1(D)$.
We will show that when $g\in B_{H^1(D)}(g_j,2\e_0)$, that is,
$g$ is $2\e_1$-close to the function $g_j$
in $H^1(D)$, the inverse map of $F_1$ can be written as a limit
$(F_1^{-1}(g),g)=\lim_{m\to \infty} \mathcal H_j^{\circ m}(v_j,g)$ in $H^1(D)^2$, 
where
  \ba
\mathcal H_j \left(\begin{array} {c}u \\ g \end{array}\right) 
=
 \left(\begin{array} {c}
 u-A_{v_j}^{-1}(F_1(u)-F_1(v_j))+A_{v_j}^{-1}(g-g_j)\\ 
  g \end{array}\right).
 \ea
That is, near $g_j$ we can approximate $F_1^{-1}$ as a composition $\mathcal H_j^{\circ m}$ of $2m$ layers of neural operators.

To glue the local inverse maps together, we use a partition of unity in the function space $\mathcal Y$
given by integral neural operators
 \ba
\Phi_{\vec i}(v,w)=\pi_1\circ \phi_{{\vec i},1}\circ \phi_{{\vec i},2}\circ \dots\circ \phi_{{\vec i},\ell_0}(v,w),\quad\hbox{where} \quad \phi_{{\vec i},\ell}(v,w)=(F_{y_\ell,s({\vec i},\ell),\epsilon_1}(v,w),w),
\ea
$\pi_1(v,w)=v$ maps a pair $(v,w)$ to the first function $v$, and ${\vec i}$ belongs to a finite index set
$\mathcal I\subset \mathbb Z^{\ell_0}$, $\epsilon_1>0$ and $y_\ell \in D$ ($\ell=1,...,\ell_0$), where
$s({\vec i},\ell)\coloneqq i_\ell \epsilon_1.$
Here, $F_{z,s,h}(v,w)$ are integral neural operators with distributional kernels
\ba
F_{z,s,h}(v,w)(x)=\int_D k_{z,s,h}(x,y,v(x),w(y))dy,
\ea
where $
k_{z,s,h}(x,y,v(x),w(y))=v(x){\bf 1}_{[s-\frac 12h,s+\frac 12h)}(w(y))\delta(y-z)$, ${\bf 1}_A$ is the indicator function of a set $A$ and $y\mapsto\delta(y-z)$ is the Dirac
delta distribution at the point $z\in D$. 

Using these, we can write the inverse of $F_1$ at $g\in\mathcal Y$ as 
 \beq\label{inverse as a limit}
   F_1^{-1}(g)=\lim_{m\to \infty} \sum_{{\vec i}\in\mathcal I}\Phi_{\vec i}   
  \mathcal H_{j({\vec i})}^{\circ m} \left(\begin{array} {c}v_{j({\vec i})} \\ g \end{array}\right),
 \eeq
where $j({\vec i})\in \{1,2,\dots,J\}$ are suitably chosen and the limit is taken in the norm topology of $H^1(D)$. This result is summarized by the following theorem,
a modified version of Theorem \ref{thm: invertibility} where the inverse operator $ F_1^{-1}$ in \eqref{inverse as a limit} have refined the partition of unity $\Phi_{\vec i}$ so that we use
indexes $\vec i\in \mathcal I\subset \mathbb Z^{\ell_0}$ instead of $j\in \{1,\dots,J\}$.

\begin{theorem}\label{thm: invertibility-app} 
Assume that $F_1$ satisfies the above assumptions  \eqref {eq: kernel k}
 and \eqref{A is injective} and that 
 $F_1:H^1(D)\to H^1(D)$ is a bijection.
Let $\mathcal Y\subset  \sigma_a(\overline B_{C^{1,\alpha}(\overline D)}(0,R))$ be a compact subset the Sobolev space $H^1(D)$, where  $\alpha>0$ and $a\ge 0$.
Then the inverse of $F_1:H^1(D)\to H^1(D)$ in $\mathcal Y$ can 
written as a limit \eqref{inverse as a limit} that is, as a limit
of integral neural operators.

\end{theorem}

Observe that Theorem \ref{thm: invertibility-app} includes the case where $a=1$,
in which case $\sigma_a=Id$ and $ \mathcal Y\subset \sigma_a(\overline B_{C^{1,\alpha}(\overline D)}(0,R))=
\overline B_{C^{1,\alpha}(\overline D)}(0,R))$.
We note that when $\sigma_a$ is a leaky ReLU-function with parameter $a>0$,
Theorem \ref{thm: invertibility-app} can be applied to compute the inverse of $\sigma_a\circ F_1$
given by $F_1^{-1}\circ \sigma_a^{-1}$, where
$\sigma_a^{-1}=\sigma_{1/a}$. Note that the assumption that 
$ \mathcal Y\subset \sigma_a(\overline B_{C^{1,\alpha}(\overline D)}(0,R))$ makes it possible to apply Theorem \ref{thm: invertibility-app} in the case when one trains deep neural networks 
having layers $\sigma_a\circ F_1$ and the parameter $a$ of the leaky ReLU-function is a free parameter which is also trained.

\begin{proof}
As the operator $F_1$ can be multiplied by function $W(x)^{-1}$, it is sufficient to consider  the case when $W(x)=1$.

Below, we use the fact that, because $D\subset \R$, Sobolev's embedding theorem yields that
the embedding $H^1(D)\to C(\overline D)$ is bounded and there is $C_S>0$ such that 
\beq
\|u\|_{C(\overline D)}\leq C_S\|u\|_{H^1(D)}.
\eeq
For clarity, we denote the norm of $C(\overline D)$ by $\|u\|_{L^\infty(D)}$.

Next we consider the Frechet derivatives of $F_1$.
We recall that the 1st Frechet derivative of $F_1$ at $u_0$ is the operator $A_{u_0}$.
The 2nd Frechet derivative of $F_1$ at $u_0$  to directions $w_1$ and $w_2$ is 
\ba
D^2F_1|_{u_0}(w_1,w_2)&=&
\int_D 2\frac {\p k}{\p u}(x,y,u_0(y))w_1(y)w_2(y) dy
+
\int_D{u_0}(y)\,\frac {\p k^2}{\p u^2}(x,y,u_0(y))w_1(y)w_2(y) dy\\
&=&
\int_D p(x,y)w_1(y)w_2(y) dy,
\ea
where
\beq
p(x,y)&=&2\frac {\p k}{\p u}(x,y,u_0(y)) 
+
u_0(y)\frac {\p k^2}{\p u^2}(x,y,u_0(y)),
\eeq
and
\beq
\frac \p{\p x} p(x,y)&=&2\frac {\p^2 k}{\p u\p x}(x,y,u_0(y)) 
+
u_0(y)\frac {\p k^3}{\p u^2\p x}(x,y,u_0(y)).
\eeq
Thus,
\beq\label{estimate for 2nd derivative}
\\
\nonumber
\|D^2F_1|_{u_0}(w_1,w_2)\|_{H^1(D)}&\le & 3|D|^{1/2}
\|k\|_{C^3(D\times D\times \R)}(1+\|u_0\|_{L^\infty(D)})\|w_1\|_{L^\infty(D)}\|w_2\|_{L^\infty(D)}.
\eeq

When we freeze the function $u$ in kernel $k$ to be $u_0$, we denote
\ba
K_{u_0}v(x)&=&\int_Dk(x,y,u_0(y)) v(y)dy.
\ea

\begin{lemma}\label{lemma on Lip}For $u_0,u_1\in C(\overline D)$  we have 
\ba
\|K_{u_1}-K_{u_0}\|_{L^2(D)\to H^1(D)}
\leq \| k\|_{C^2(D\times D\times \R)}|D|  \|u_1-u_0\|_{L^\infty(D)}.
\ea
and
\beq\label{Alip}
\|A_{u_1}-A_{u_0}\|_{L^2(D)\to H^1(D)}
\leq2\| k\|_{C^2(D\times D\times \R)}|D| (1+ \|u_0\|_{L^\infty(D)}) \|u_1-u_0\|_{L^\infty(D)}.
\eeq

\end{lemma}
\begin{proof}
Denote
\ba
M_{u_0}v(x)&=&\int_D u_0(y)\,\frac {\p k}{\p u}(x,y,u_0(y)) v(y)dy,\\
N_{u_1,u_2}v(x)&=&\int_D u_1(y)\,\frac {\p k}{\p u}(x,y,u_2(y)) v(y)dy.
\ea
We have 
\ba
M_{u_2}v-M_{u_1}v= (N_{u_2,u_2}v-N_{u_2,u_1}v)+(N_{u_2,u_1}v-N_{u_1,u_1}v).
\ea
By Schur's test for continuity of integral operators,
\ba
\|K_{u_0}\|_{L^2(D)\to L^2(D)}&\leq& \bigg(\sup_{x\in D}\int_D |k(x,y,u_0(y))|dy\bigg)^{1/2}
\bigg(\sup_{y\in D}\int_D |k(x,y,u_0(y))|dx\bigg)^{1/2}
\\
&\leq& \| k\|_{C^0(D\times D\times \R)},
\ea
and
\ba
&&\|M_{u_0}\|_{L^2(D)\to L^2(D)}
\\
&\leq& \bigg(\sup_{x\in D}\int_D |u_0(y)\,\frac {\p k}{\p u}(x,y,u_0(y))|dy\bigg)^{1/2}
\bigg(\sup_{y\in D}\int_D |u_0(y)\,\frac {\p k}{\p u}(x,y,u_0(y))|dx\bigg)^{1/2}
\\
&\leq& \| k\|_{C^1(D\times D\times \R)}\|u\|_{C(\overline D)},
\ea
and
\ba
&&\|K_{u_2}-K_{u_1}\|_{L^2(D)\to L^2(D)}
\\&\leq& \bigg(\sup_{x\in D}\int_D |k(x,y,u_2(y))-k(x,y,u_1(y))|dy\bigg)^{1/2}
\\
&& \hspace{2cm} \times
\bigg(\sup_{y\in D}\int_D |k(x,y,u_2(y))-k(x,y,u_1(y))|dx\bigg)^{1/2}
\\&\leq& 
\bigg(\| k\|_{C^1(D\times D\times \R)}
\int_D |u_2(y)-u_1(y))|dy\bigg)^{1/2}
\\
&& \hspace{2cm} \times
\bigg( \| k\|_{C^1(D\times D\times \R)}\sup_{y\in D}\int_D|u_2(y)-u_1(y))|dx\bigg)^{1/2}
\\
&\leq& \| k\|_{C^1(D\times D\times \R)} \bigg(\int_D |u_2(y)-u_1(y))|dy\bigg)^{1/2}
\bigg(\sup_{y\in D}\int_D |u_2(y)-u_1(y))|dx\bigg)^{1/2}
\\
&\leq& \| k\|_{C^1(D\times D\times \R)}\bigg(|D|^{1/2}  \|u_2-u_1\|_{L^2(D)}\bigg)^{1/2}\bigg( |D|
\sup_{y\in D}|u_2(y)-u_1(y))|\bigg)^{1/2}
\\
&\leq& \| k\|_{C^1(D\times D\times \R)}|D|^{3/4}  \|u_2-u_1\|_{L^2(D)}^{1/2}\|u_2-u_1\|_{L^\infty(D)}^{1/2}
\\
&\leq& \| k\|_{C^1(D\times D\times \R)}|D| \|u_2-u_1\|_{L^\infty(D)},
\ea
and
\ba
&&\|N_{u_2,u_2}-N_{u_2,u_1}\|_{L^2(D)\to L^2(D)}
\\&\leq& \bigg(\sup_{x\in D}\int_D |u_2(y)k(x,y,u_2(y))-u_2(y)k(x,y,u_1(y))|dy\bigg)^{1/2}
\\ 
& &\quad \quad \times
\bigg(\sup_{y\in D}\int_D |u_2(y)k(x,y,u_2(y))-u_2(y)k(x,y,u_1(y))|dx\bigg)^{1/2}
\\
&\leq& \| k\|_{C^1(D\times D\times \R)}|D|^{3/4} \|u_2\|_{C^0(D)} \|u_2-u_1\|_{L^2(D)}^{1/2}\|u_2-u_1\|_{L^\infty(D)}^{1/2}
\\
&\leq& \| k\|_{C^1(D\times D\times \R)}|D| \|u_2\|_{C^0(D)} \|u_2-u_1\|_{L^\infty(D)},
\ea
and
\ba
&&\|N_{u_2,u_1}-N_{u_1,u_1}\|_{L^2(D)\to L^2(D)}
\\&\leq& \bigg(\sup_{x\in D}\int_D |(u_2(y)-u_1(y))k(x,y,u_1(y))|dy\bigg)^{1/2}
\\ 
& &\quad \quad \times
\bigg(\sup_{y\in D}\int_D |(u_2(y)-u_1(y))k(x,y,u_1(y))|dx\bigg)^{1/2}
\\
&\leq& \| k\|_{C^0(D\times D\times \R)}|D|  \|u_2-u_1\|_{L^\infty(D)},
\ea
so that
\ba
&&\|M_{u_2}-M_{u_1}\|_{L^2(D)\to L^2(D)}
\\
&\leq& \| k\|_{C^1(D\times D\times \R)}|D| (1+ \|u_2\|_{C^0(D)}) \|u_2-u_1\|_{L^\infty(D)}.
\ea
Also,
 when $D_xv=\frac {dv}{dx}$,
\ba
&&\|D_x\circ K_{u_0}\|_{L^2(D)\to L^2(D)}
\\
&\leq& \bigg(\sup_{x\in D}\int_D |D_xk(x,y,u_0(y))|dy\bigg)^{1/2}
\bigg(\sup_{y\in D}\int_D |D_xk(x,y,u_0(y))|dx\bigg)^{1/2}
\\
&\leq& \| k\|_{C^1(D\times D\times \R)},
\ea
and 
\ba
&&\|D_x\circ K_{u_1}-D_x\circ  K_{u_0}\|_{L^2(D)\to L^2(D)}
\\&\leq& \bigg(\sup_{x\in D}\int_D |D_xk(x,y,u_1(y))-D_xk(x,y,u_0(y))|dy\bigg)^{1/2}
\\
&& \hspace{2cm} \times
\bigg(\sup_{y\in D}\int_D |D_xk(x,y,u_1(y))-D_xk(x,y,u_0(y))|dx\bigg)^{1/2}
\\&\leq& \bigg(
\| k\|_{C^2(D\times D\times \R)}
\int_D |u_1(y)-u_0(y))|dy\bigg)^{1/2}
\\
&& \hspace{2cm} \times
\bigg(
\| k\|_{C^2(D\times D\times \R)}
\sup_{y\in D}\int_D |u_1(y)-u_0(y))|dx\bigg)^{1/2}
\\
&\leq& \| k\|_{C^2(D\times D\times \R)} \bigg(\int_D |u_1(y)-u_0(y))|dy\bigg)^{1/2}
\bigg(\sup_{y\in D}\int_D |u_1(y)-u_0(y))|dx\bigg)^{1/2}
\\
&\leq& \| k\|_{C^2(D\times D\times \R)}\bigg(|D|^{1/2}  \|u_1-u_0\|_{L^2(D)}\bigg)^{1/2}\bigg( |D|
\sup_{y\in D}|u_1(y)-u_0(y))|\bigg)^{1/2}
\\
&\leq& \| k\|_{C^2(D\times D\times \R)}|D|^{3/4}  \|u_1-u_0\|_{L^2(D)}^{1/2}\|u_1-u_0\|_{L^\infty(D)}^{1/2}
\\
&\leq& \| k\|_{C^2(D\times D\times \R)}|D| \|u_1-u_0\|_{L^\infty(D)}.
\ea
Thus,
\ba
\|K_{u_0}\|_{L^2(D)\to H^1(D)}&\leq& \| k\|_{C^1(D\times D\times \R)},
\ea
and
\ba
\|M_{u_0}\|_{L^2(D)\to H^1(D)}&\leq& \|u_0\|_{C^0(D)} \| k\|_{C^1(D\times D\times \R)},
\ea
and
\ba
\|K_{u_1}-K_{u_0}\|_{L^2(D)\to H^1(D)}
&\leq& \| k\|_{C^2(D\times D\times \R)}|D|  \|u_1-u_0\|_{L^\infty(D)}.
\ea
Similarly,
\ba
\|M_{u_1}-M_{u_0}\|_{L^2(D)\to H^1(D)}
&\leq&\| k\|_{C^2(D\times D\times \R)}|D| (1+ \|u_2\|_{C^0(D)}) \|u_1-u_0\|_{L^\infty(D)}.
\ea
As $A_{u_1}=K_{u_1}+M_{u_1}$, the claim follows.
\end{proof}

As the embedding $H^1(D)\to C(\overline D)$ is bounded and 
has norm $C_S$,  Lemma \ref{lemma on Lip} implies that 
 for all $R>0$ there is 
 \ba
 C_L(R)=2\| k\|_{C^2(D\times D\times \R)}|D| (1+ C_SR),
 \ea such that
the map,
\beq
u_0 \mapsto DF_1|_{u_0},\quad u_0\in \overline B_{H^1}(0,R),
\eeq
is a Lipschitz map $\overline B_{H^1}(0,R)\to \mathcal{L}(H^1(D),H^1(D))$ with Lipschitz constant
$C_L(R)$, that is,
\beq\label{eq: F1 Lip}
\|DF_1|_{u_1}-DF_1|_{u_2}\|_{H^1(D)\to H^1(D)}\leq C_L(R)\|u_1-u_2\|_{H^1(D)}.
\eeq

As $u_0\mapsto A_{u_0}=DF_1|_{u_0}$  is continuous,
the inverse $A_{u_0}^{-1}:H^1(D)\to H^1(D)$ exists for all $u_0\in C(\overline D)$, and
 the  embedding $H^1(D)\to C(\overline D)$ is compact, we have that 
 for all $R>0$ there is $C_B(R)>0$ such that
\beq
\|A_{u_0}^{-1}\|_{H^1(D)\to H^1(D)}\leq C_B(R),\quad \hbox{for all }u_0\in \overline B_{H^1}(0,R).
\eeq

Let $R_1,R_2>0$ be  such
that  $\mathcal Y\subset \overline B_{H^1}(0,R_1)$ and
$X=F_1^{-1}(\mathcal Y) \subset \overline B_{H^1}(0,R_2)$.
Below, we denote $C_L=C_L(2R_2)$ and $C_B=C_B(R_2)$.

Next we consider inverse of $F_1$ in $\mathcal Y$. To this end, let us consider 
$\e_0>0$, which we choose later to be small enough.
As $\mathcal Y\subset \overline B_{H^1}(0,R)$ is  compact 
there are a finite number of elements $g_j=F_1(v_j)\in \mathcal Y$, where  $v_j\in X$, $j=1,2,\dots,J$
such that
\ba
\mathcal  Y\subset \bigcup_{j=1}^J B_{H^1(D)}(g_j,\e_0).
\ea

We observe that for $u_0,u_1\in X$,
\ba
A_{u_1}^{-1}-A_{u_0}^{-1}=A_{u_1}^{-1}(A_{u_1}-A_{u_0})A_{u_0}^{-1},
\ea
and hence the Lipschitz constant of $A_{\cdot }^{-1}:u\mapsto A_{u}^{-1}$, $X\to  \mathcal{L}(H^1(D), H^{1}(D))$
satisfies
\beq
Lip(A_{\cdot }^{-1})\leq C_A
=C_{B}^2 C_L,
\eeq
see \eqref{Alip}.

Let us consider a fixed $j$ and $g_j\in \mathcal Y$. When $g$ satisfies
\beq
\|g-g_j\|_{H^1(D)}<2\e_0,
\eeq
 the equation
 \ba
 F_1(u)=g,\quad u\in X,
 \ea
 is equivalent to the fixed point equation 
  \ba
u=u-A_{v_j}^{-1}( F_1(u)-F_1(v_j))+A_{v_j}^{-1}( g-g_j),
 \ea 
that  is equivalent to the fixed point equation 
 \ba
H_j(u)= u,
\ea
for the function $H_j:H^1(D)\to H^1(D)$,
\ba
H_j(u)= u-A_{v_j}^{-1}(F_1(u)-F_1(v_j))+A_{v_j}^{-1}(g-g_j).
 \ea
Note that $H_j$ depends on $g$,
 and thus we later denote $H_j=H_j^g$. We observe that
 \beq
 H_j(v_j)=v_j+A_{v_j}^{-1}(g-g_j).
 \eeq
 Let $u,v\in  \overline B_{H^1}(0,2R_2)$.
  We have 
 \ba
 F_1(u)=F_1(v)+A_{v}( u-v)+B_v( u-v),\quad \|B_v( u-v)\|\leq C_0\| u-v\|^2,
 \ea
where, see
 \eqref{estimate for 2nd derivative},
\ba
 C_0= 3|D|^{1/2} \|k\|_{C^3(D\times D\times \R)}(1+2C_SR_2)C_S^2,
 \ea
 so that  for $u_1,u_2\in  \overline B_{H^1}(0,2R_2)$,
 \ba
& &u_1-u_2-A^{-1}_{v_j}(F_1(u_1)-F_1(u_2))\\
&=&u_1-u_2-A^{-1}_{u_2}(F_1(u_1)-F_1(u_2))-(A^{-1}_{u_2}-A^{-1}_{v_j})(F_1(u_1)-F_1(u_2)),
  \ea
 and 
\ba
& &\|u_1-u_2-A^{-1}_{u_2}(F_1(u_1)-F_1(u_2))\|_{H^1(D)}
\\
&=&
\|A^{-1}_{u_2}(B_{u_2}( u_1-u_2))\|_{H^1(D)}
\\
&\le&
\|A^{-1}_{u_2}\|_{H^1(D)\to H^1(D)} \|B_{u_2}( u_1-u_2)\|_{H^1(D)}
\\
&\le&
\|A^{-1}_{u_2}\|_{H^1(D)\to H^1(D)} C_0 \|u_1-u_2\|_{H^1(D)}^2,
\\
&\le&
C_B C_0 \|u_1-u_2\|_{H^1(D)}^2,
\ea
and 
\ba
& &\| (A^{-1}_{u_2}-A^{-1}_{v_j})(F_1(u_1)-F_1(u_2))\|_{H^1(D)}\\
& \le &
\| A^{-1}_{u_2}-A^{-1}_{v_j}\|_{H^1(D)\to H^1(D)} \|F_1(u_1)-F_1(u_2)\| _{H^1(D)}\\
& \le &
Lip_{ \overline B_{H^1}(0,2R_2)\to H^1(D)}( A^{-1}_{\cdot}) \|u_2-v_j\|  Lip_{ \overline B_{H^1}(0,2R_2)\to H^1(D)}(F_1) \|u_2-u_1\|_{H^1(D)} 
\\
& \le &
C_A \|u_2-v_j\| (C_{B} + 4C_0 R_2) \|u_2-u_1\|_{H^1(D)}, 
\ea
see \eqref{A-operator},
and hence, when $\|u-v_j\|\le r\le R_2,$
\ba
& &\|H_j(u_1)-H_j(u_2)\|_{H^1(D)}\\
&\le &\|u_1-u_2-A^{-1}_{v_j}(F_1(u_1)-F_1(u_2))\|_{H^1(D)}
\\
&\le &\|u_1-u_2-A^{-1}_{u_2}(F_1(u_1)-F_1(u_2))\|_{H^1(D)}+\| (A^{-1}_{u_2}-A^{-1}_{v_j})(F_1(u_1)-F_1(u_2))\|_{H^1(D)}
\\
&\le&
\bigg(C_B C_0 (\|u_1-v_j\|_{H^1(D)}+\|u_2-v_j\|_{H^1(D)})+ C_A (C_B + 4C_0 R_2) \|u_2-v_j\| \bigg) \|u_2-u_1\| _{H^1(D)}
\\
&\le&
C_H r \|u_2-u_1\| _{H^1(D)},
\ea
where
\ba
C_H=2C_BC_0 +C_A(C_{B} + 4C_0 R_2).
\ea
 
We now choose 
    \ba
  r= \min(\frac 1{2C_H},R_2).
  \ea
We consider
   \ba
 \e_0\le  \frac 1{8C_B} \frac 1{2C_H}.
  \ea 
  Then, we have
  \ba
  r\ge 2 C_B \e_0/(1-C_Hr).
  \ea 
Then, we have that Lip${}_{ \overline B_{H^1}(0,2R_2)\to H^1(D)}(H_j)\le a=C_Hr<\frac 12$, and 
\ba
  r\ge \|A_{v_j}^{-1}\|_{H^1(D)\to H^1(D)} \|g-g_j\|_{H^1(D)}/(1-a),
\ea 
 and for all $u\in  \overline B_{H^1}(0,R_2)$ such that  $\|u-v_j\|\le r$, 
 we have 
 $\|A_{v_j}^{-1}(g-g_j)\|_{H^1(D)}\le(1-a)r$.
  Then,
 \ba
 \|H_j(u)-v_j\|_{H^1(D)}&\le&
 \|H_j(u)-H_j(v_j)\|_{H^1(D)}+ \|H_j(v_j)-v_j\|_{H^1(D)}
 \\
 &\le&  a\|u-v_j\|_{H^1(D)}+ \|v_j+A_{v_j}^{-1}(g-g_j)-v_j\|_{H^1(D)}
 \\
 &\le&  ar+ \|A_{v_j}^{-1}(g-g_j)\|_{H^1(D)}\le r,
 \ea
 that is, $H_j$ maps $\overline B_{H^1(D)}(v_j,r)$ to itself. By Banach fixed point
 theorem, $H_j:\overline B_{H^1(D)}(v_j,r)\to \overline B_{H^1(D)}(v_j,r)$ has a fixed point.
 
Let us denote
\ba
\mathcal H_j \left(\begin{array} {c}u \\ g \end{array}\right) =
 \left(\begin{array} {c}
 H_j^g(u)\\ 
  g \end{array}\right)
=
 \left(\begin{array} {c}
 u-A_{v_j}^{-1}(F_1(u)-F_1(v_j))+A_{v_j}^{-1}(g-g_j)\\ 
  g \end{array}\right).
\ea
 
 By the above, when we choose 
 $\e_0$ to have a value
  \ba
 \e_0< \frac 1{8C_B} \frac 1{2C_H},
  \ea 
the map $F_1$ has a right inverse map $ \mathcal R_j$ in $B_{H^1}(g_j,2\e_0)$,
 that is,
 \beq
  F_1(\mathcal R_j (g))= g,\quad \hbox{for }g\in B_{H^1}(g_j,2\e_0),
 \eeq
 it holds that $ \mathcal R_j:B_{H^1}(g_j,2\e_0)\to \overline B_{H^1(D)}(v_j,r)$, 
 and by Banach fixed point theorem it is given by the limit
 \beq
 \mathcal R_j(g)=\lim_{m\to \infty} w_{j,m},\quad g\in B_{H^1}(g_j,2\e_0),
 \eeq
 in $H^1(D)$, where 
 \beq
& & w_{j,0}=v_j,\\
  & &w_{j,m+1}=H^g_j(w_{j,m}).
  \eeq
  We can write for $g\in B_{H^1}(g_j,2\e_0)$,
    \ba
 \left(\begin{array} {c}
  \mathcal R_j(g)\\ 
  g \end{array}\right)=\lim_{m\to \infty}  \mathcal H_j^{\circ m}
   \left(\begin{array} {c} v_j\\ 
  g \end{array}\right),
\ea
where the limit takes space in $H^1(D)^2$ and
   \beq
\mathcal H_j^{\circ m}=\mathcal H_j\circ \mathcal H_j\circ \dots \circ \mathcal H_j,
  \eeq
is the composition of $m$ operators $\mathcal H_j$.
This implies that $ \mathcal R_j$ can be written as a limit of finite
iterations of neural operators $H_j$ (we will consider how the operator $A_{v_j}^{-1}$ 
can be written as a neural operator below).

As $\mathcal Y\subset \sigma_a(\overline B_{C^{1,\alpha}(\overline D)}(0,R))$,
there are finite number
of points $y_\ell \in D$,  $\ell=1,2,\dots, \ell_0$ and $\e_1>0$ such that 
the sets 
\ba
Z(i_1,i_2,\dots,i_{\ell_0})=\{g\in \mathcal{Y}:\ (i_\ell-\frac 12)\e_1 \le g(y_\ell)<(i_\ell+\frac 12)\e_1,\ \hbox{for all }\ell\},
\ea
where $i_1,i_2,\dots,i_{\ell_0}\in \mathbb Z$, satisfy the condition
\beq\label{Z-property}\\
\nonumber
\hbox{If $(Z(i_1,i_2,\dots,i_{\ell_0})\cap \mathcal Y)\cap B_{H^1(D)}(g_j,\e_0)\not =\emptyset$
then $Z(i_1,i_2,\dots,i_{\ell_0})\cap \mathcal Y\subset  B_{H^1(D)}(g_j,2\e_0)$}.\hspace{-1cm}
\eeq
To show \eqref{Z-property}, we will below use the mean value theorem for function $g=\sigma_a\circ v \in  \mathcal Y$,
where $v\in C^{1,\alpha}(\overline D).$ 
 First, let us consider the case when
  the parameter $a$ of the leaky ReLU function $\sigma_a$
is strictly positive. 
Without loss of generality, we can assume that $D=[0,1]$ and
$y_\ell=h\ell$, where $h=1/\ell_0$ and $\ell=0,1,\dots,\ell_0$.
 We consider $g\in  \mathcal Y\cap  Z(i_1,i_2,\dots,i_{\ell_0})\subset \sigma_a(\overline B_{C^{1,\alpha}(\overline D)}(0,R))$ 
of the form $g=\sigma_a\circ v$.
As $a$ is non-zero, the inequality $(i_\ell-\frac 12)\e \le g(y_\ell)<(i_\ell+\frac 12)\e$
is equivalent to $\sigma_{1/a}((i_\ell-\frac 12)\e) \le v(y_\ell)<\sigma_{1/a}((i_\ell+\frac 12)\e)$,
and thus
\beq\label{values of with small errors}
\sigma_{1/a}(i_\ell\e)-A\e \le v(y_\ell)<\sigma_{1/a}(i_\ell\e)+A\e,
\eeq
where $A=\frac{1}{2}\max(1,a,1/a)$,
that is, for $g=\sigma_a(v)\in Z(i_1,i_2,\dots,i_{\ell_0})$ the values $v(y_\ell)$ are known within small errors.
By applying mean value theorem on the interval
$[(\ell-1)h,\ell h]$ for function $v$ we see  that there is $x'\in [(\ell-1)h,\ell h]$
such that 
\ba
\frac {dv}{dx}(x')=\frac {v(\ell h)-v((\ell-1)h)}h,
\ea
and thus by \eqref{values of with small errors},
\beq\label{values of derivative with small errors}
|\frac {dv}{dx}(x')-d_{\ell,\vec i}|\leq 2A\frac {\e_1}h,
\eeq
where 
\beq\label{values of derivative with small errors2}
d_{\ell,\vec i}=\frac 1h (\sigma_{1/a}(i_\ell\e_1)-\sigma_{1/a}((i_\ell-1)\e_1)),
\eeq
Observe that these estimates are useful when $\e_1$ is much smaller that $h$.
As $g=\sigma_a\circ v\in  \mathcal Y\subset \sigma_a(\overline B_{C^{1,\alpha}(\overline D)}(0,R))$, we have $v\in \overline B_{C^{1,\alpha}(\overline D)}(0,R)$, so that 
$\frac {dv}{dx}\in \overline B_{C^{0,\alpha}(\overline D)}(0,R)$
satisfies
\eqref{values of derivative with small errors} implies that
\beq\label{values of derivative with small errors3}
|\frac {dv}{dx}(x)-d_{\ell,\vec i}|\leq 2A\frac {\e_1}h+Rh^\alpha,
\quad \hbox{for all $x\in [(\ell-1)h,\ell h]$}.
\eeq
Moreover, \eqref{values of with small errors} and
 $v\in \overline B_{C^{1,\alpha}(\overline D)}(0,R)$ imply
\beq\label{values of v with small errors2}
| v(x)-\sigma_{1/a}(i_\ell\e_1)|<A\e_1+Rh,
\eeq
for all $x\in [(\ell-1)h,\ell h]$.

Let $\e_2=\e_0/A.$
When we first choose $\ell_0$ to be large enough (so that $h=1/\ell_0$  is small) and then
$\e_1$ to be small enough, we may assume that
\beq
\max(2A\frac {\e_1}h+Rh^\alpha,A\e_1+Rh)<\frac 1{8}\e_2.
\eeq
Then for any two functions $g,g'\in  \mathcal Y\cap  Z(i_1,i_2,\dots,i_{\ell_0})\subset \sigma_a(\overline B_{C^{1,\alpha}(\overline D)}(0,R))$ 
of the form $g=\sigma_a\circ v,g'=\sigma_a\circ v'$ 
the inequalities 
\eqref{values of derivative with small errors3} and \eqref{values of v with small errors2} 
imply
 \beq\label{values of derivative with small errors4}
& & |\frac {dv}{dx}(x)-\frac {dv'}{dx}(x)|<\frac 14\e_2,
\\
\nonumber& & |v(x)-v'(x)|<\frac 14\e_2,
 \eeq
 for all $x\in D$. As $v,v'\in \overline B_{C^{1,\alpha}(\overline D)}(0,R)$, this implies 
\ba
\|v-v'\|_{{C^{1}(\overline D)}}<\frac 12\e_2,
\ea
 As the embedding
$C^1(\overline D)\to H^1(D)$ is continuous and has norm less than 2 on the interval $D=[0,1]$, we see that 
\ba
\|v-v'\|_{H^{1}(\overline D)}<\e_2,
\ea
and thus
\begin{equation}
\|g-g'\|_{H^{1}(\overline D)}<A\e_2=\e_0.
\label{estimate-g-g-prime}
\end{equation}

To prove \eqref{Z-property}, we assume that $(Z(i_1,i_2,\dots,i_{\ell_0})\cap \mathcal Y)\cap B_{H^1(D)}(g_j,\e_0)\not =\emptyset$, and $g\in Z(i_1,i_2,\dots,i_{\ell_0})\cap \mathcal Y$.
By the assumption, there exists $g^{\ast} \in (Z(i_1,i_2,\dots,i_{\ell_0})\cap \mathcal Y)\cap B_{H^1(D)}(g_j,\e_0)$.
Using (\ref{estimate-g-g-prime}), we have
\[
\left\| g-g_j \right\|_{H^1(D)}
\leq 
\left\| g-g^{\ast} \right\|_{H^1(D)}
+
\left\| g^{\ast}-g_j \right\|_{H^1(D)}
\leq 2 \epsilon_0.
\]
Thus, $g \in B_{H^1(D)}(g_j,2\e_0)$, which implies that the property \eqref{Z-property} follows.

We next consider the case when the parameter $a$ of the leaky relu function $\sigma_a$
is zero. Again, we assume that $D=[0,1]$ and
$y_\ell=h\ell$, where $h=1/\ell_0$ and $\ell=0,1,\dots,\ell_0$. We consider $g\in  \mathcal Y\cap  Z(i_1,i_2,\dots,i_{\ell_0})\subset \sigma_a(\overline B_{C^{1,\alpha}(\overline D)}(0,R))$ 
of the form $g=\sigma_a(v)$ and  an interval
$[\ell_1h,(\ell_1+1)h]\subset D$, where $1\le \ell_1\le \ell_0-2$. 
We will consider four cases.
First, if $g$  does not obtain  the value zero on the interval
$[\ell_1h,(\ell_1+1)h]$ 
the mean value theorem implies that there is $x'\in [\ell_1h,(\ell_1+1)h]$
such that $\frac {dg}{dx}(x')=\frac {dv}{dx}(x')$ is equal to $d=(g(\ell_1h)-g([(\ell_1-1)h))/h$.
Second, if $g$  does not obtain the value zero on either of the intervals $[(\ell_1-1)h,\ell_1h]$ or
$[(\ell_1+1)h,(\ell_1+2)h]$, we can use the mean value theorem to estimate
the derivatives of $g$ and $v$ at some point of these intervals similarly to the first case.
Third, if  $g$ does not vanish identically on the interval
$[\ell_1h,(\ell_1+1)h]$ but it obtains the value zero on the both intervals $[(\ell_1-1)h,\ell_1h]$ and 
$[(\ell_1+1)h,(\ell_1+2)h]$, the function $v$ has two zeros
on the interval $[(\ell_1-1)h,(\ell_1+2)h]$ and 
the mean value theorem implies that there
is $x'\in [(\ell_1-1)h,(\ell_1+2)h]$ such that 
$\frac {dv}{dx}(x')=0$.
Fourth, if none of the above cases are valid,
$g$ vanishes identically on the interval $[\ell_1h,(\ell_1+1)h]$.
In all these cases the fact that 
$\|v\|_{C^{1,\alpha}(\overline D)}\le R$ implies that the derivative of $g$
can be estimated on the whole interval $[\ell_1 h,(\ell_1+1)h]$
within a small error.
Using these observations we see
for any $\e_2,\e_3>0$ that if  $y_\ell \in D=[d_1,d_2]\subset \R$,  $\ell=1,2,\dots, \ell_0$ are a sufficiently dense grid in 
$D$ and $\e_1$ to be small enough, then the derivatives of any two functions $g,g'\in  \mathcal Y\cap  Z(i_1,i_2,\dots,i_{\ell_0})\subset \sigma_a(\overline B_{C^{1,\alpha}(\overline D)}(0,R))$ 
of the form $g=\sigma_a(v),g'=\sigma_a(v')$ satisfly $\|g-g'\|_{H^1( [d_1+\e_3,d_2-\e_3])}<\e_2$.  As the embedding
$C^1( [d_1+\e_3,d_2-\e_3])\to H^1( [d_1+\e_3,d_2-\e_3])$ is continuous,
\ba
& &\|\sigma_a(v)\|_{H^1( [d_1,d_1+\e_3])}\leq c_a\|v\|_{C^{1,\alpha}(\overline D)}\sqrt \e_3,\\
& &\|\sigma_a(v)\|_{H^1( [d_2-\e_3,d_2])}\leq c_a\|v\|_{C^{1,\alpha}(\overline D)}\sqrt \e_3,
\ea
and $\e_2$ and $\e_3$
can be chosen to be arbitrarily small, we see that the property \eqref{Z-property} follows.
Thus the property \eqref{Z-property} is shown in all cases.

By our assumptions $\mathcal Y\subset \sigma_a(B_{C^{1,\alpha}(\overline D)}(0,R))$ and hence $g\in \mathcal Y$ implies that
$\|g\|_{C(\overline D)}\leq AR.$ This implies that 
$\mathcal Y\cap Z(i_1,i_2,\dots,i_{\ell_0})$ is empty
if there is $\ell$ such that $|i_\ell|>2AR/\e_1+1$.
Thus, there is a finite set $\mathcal I\subset \mathbb Z^{\ell_0}$ such that 
\beq
& &\mathcal{Y}\subset \bigcup_{\vec{i}\in \mathcal{I}}Z(\vec{i}),
\\ 
&& 
Z(\vec{i})\cap {\mathcal{Y}}\not =\emptyset ,\quad\hbox{for all }
\vec{i}\in{\mathcal{I}},
\eeq
where we use notation $\vec i=(i_1,i_2,\dots,i_{\ell_0})\in \mathbb Z^{\ell_0}$.
On the other hand, we have chosen $g_j\in \mathcal Y$
such that $B_{H^1(D)}(g_j,\e_0)$, $j=1,\dots,J$ cover 
$\mathcal Y$. This implies that for all 
$\vec i\in{\mathcal I}$ there is $j=j(\vec i)\in\{1,2,\dots,j\}$ such that
there exists $g\in Z(\vec i)\cap 
B_{H^1(D)}(g_j,\e_0)$. By \eqref{Z-property},
this implies that
\beq\label{2e0 covering}
Z(\vec i)\subset 
B_{H^1(D)}(g_{j(\vec i)},2\e_0).
\eeq
Thus, we see that $Z(\vec i)$, $\vec i\in{\mathcal I}$ 
is a disjoint covering of $\mathcal Y$, and by
\eqref{2e0 covering}, 
in each set $Z(\vec i)\cap \mathcal Y$,
$\vec i\in{\mathcal I}$ the map $g\to \mathcal R_j(g)$ we have constructed 
a right inverse of the map $F_1$.

 Below, we denote $s(\vec i,\ell)=i_\ell\e_1$. 
Next we construct a partition of unity in $\mathcal Y$  using maps
\ba
F_{z,s,h}(v,w)(x)=\int_D k_{z,s,h}(x,y,v(x),w(y))dy,
\ea
where 
\ba
k_{z,s,h}(x,y,v(x),w(y))=v(x){\bf 1}_{[s-\frac 12h,s+\frac 12h)}(w(y))\delta(y-z).
\ea
Then, 
\ba
F_{z,s,h}(v,w)(x)=\begin{cases} v(x),&\hbox{ if }-\frac 12h\le w(z)-s<\frac 12h,\\
\ \ \ 0, &\hbox{ otherwise.}
\end{cases} 
\ea
Next, for  all ${\vec i}\in \mathcal{I}$
we define the operator $\Phi_{\vec i}:H^1(D)\times \mathcal Y\to H^1(D)$,
\ba
\Phi_{\vec i}(v,w)=\pi_1\circ \phi_{{\vec i},1}\circ \phi_{{\vec i},2}\circ \dots\circ \phi_{{\vec i},\ell_0}(v,w),
\ea
where $\phi_{{\vec i},\ell}:H^1(D)\times \mathcal Y\to H^1(D)\times \mathcal Y$ are the maps
$$\phi_{{\vec i},\ell}(v,w)=(F_{y_\ell,s(\vec i,\ell),\e_1}(v,w),w),$$ 
 and $\pi_1(v,w)=v$ maps a pair $(v,w)$ to the first function $v$.
It satisfies
\ba
\Phi_{\vec i}(v,w)=\begin{cases} \ \ \ v,&\hbox{ if }-\frac 12\e_1\le w(y_\ell)-s(\vec i,\ell)<\frac 12\e_1\hbox{ for all $\ell$},\\
\ \ \ 0, &\hbox{ otherwise.}
\end{cases} 
\ea
Observe that here $s(\vec i,\ell)=i_\ell\e_1$ is close to the value $g_{j(\vec i)}(y_\ell)$.
Now we can write for $g\in Y$
\ba
F_1^{-1}(g)=\sum_{{\vec i}\in \mathcal I}\Phi_{\vec i}(\mathcal R_{j(\vec i)}(g),g),
\ea
with suitably chosen $j(\vec i)\in \{1,2,\dots, J\}$.

Let us finally consider $A_{u_0}^{-1}$ where $u_0\in C(\overline D)$. 
Let us denote 
$$\tilde K_{u_0}w=\int_D{u_0}(y)\,\frac {\p k}{\p u}(x,y,u_0(y))w(y) dy,
$$
and $J_{u_0}=K_{u_0}+\tilde K_{u_0}$ be the integral operator
with kernel
$$
j_{u_0}(x,y)=k(x,y,u_0(y))+{u_0}(y)\,\frac {\p k}{\p u}(x,y,u_0(y)).
$$
 We have 
 \ba
 (I+J_{u_0})^{-1}=I-J_{u_0}+J_{u_0}(I+J_{u_0})^{-1}J_{u_0},
 \ea
 so that when we write the linear bounded operator $$A_{u_0}^{-1}=(I+J_{u_0})^{-1}:H^1(D)\to H^1(D),$$ as an integral operator
 $$(I+J_{u_0})^{-1}v(x)=v+\int_Dm_{u_0}(x,y)v(y)dy,$$
 we have
 \ba
 &&(I+J_{u_0})^{-1}v(x)\\
 &=&v(x)-J_{u_0}v(x)
 \\
 &&\hspace{-5mm}
 +\int_D\left(\int_D \left\{j_{u_0}(x,y')j_{u_0}(y,y')
 +\left(\int_D j_{u_0}(x,y')m_{u_0}(y',x')j_{u_0}(x',y)dx'\right)
 \right\}dy'\right)v(y)dy \\
 &=&v(x)-\int_D\tilde j_{u_0}(x,y)v(y)dy,
  \ea
 where 
 \ba
 \tilde j_{u_0}(x,y)=-j_{u_0}(x,y)+\int_D (j_{u_0}(x,y')j_{u_0}(y,y')dy'+\int_D\int_D j_{u_0}(x,y')m_{u_0}(y',x')j_{u_0}(x',y)dx'dy'.
 \ea
 This implies that the operator $A_{u_0}^{-1}= (I+J_{u_0})^{-1}$ is a neural operator, too.
 Observe that $\tilde j_{u_0}(x,y),\p_x \tilde j_{u_0}(x,y)\in C(\overline D\times \overline D)$.

This proves  Theorem \ref{thm: invertibility}.\end{proof}

\newpage

\end{document}